%% file: paper.tex
\documentclass[11pt]{article}
\usepackage[preprint]{acl}

\usepackage{times}
\usepackage{latexsym}
\usepackage[T1]{fontenc}
\usepackage[utf8]{inputenc}
\usepackage{microtype}
\usepackage{inconsolata}
\usepackage{graphicx}
\usepackage{amsmath,amssymb,amsthm}
\usepackage{mathtools}
\usepackage{bbm}
\usepackage{booktabs}
\usepackage{multirow}
\usepackage{longtable}

\newtheorem{theorem}{Theorem}
\newtheorem{lemma}{Lemma}
\newtheorem{corollary}{Corollary}
\theoremstyle{remark}
\newtheorem{remark}{Remark}

\newcommand{\Prob}{\mathbb{P}}
\newcommand{\E}{\mathbb{E}}
\newcommand{\Var}{\operatorname{Var}}
\newcommand{\Cov}{\operatorname{Cov}}
\newcommand{\1}{\mathbbm{1}}
\newcommand{\ns}{\textsuperscript{\textdagger}}

\newcommand{\Figure}{Fig.}
\newcommand{\Table}{Tab.}
\newcommand{\Tables}{Tabs.}
\newcommand{\Section}{Sec.}
\newcommand{\Appendix}{App.}
\newcommand{\Theorem}{Thm.}
\newcommand{\Corollary}{Cor.}
\newcommand{\Lemma}{Lem.}

\graphicspath{{figures/}}

\usepackage{xspace}
\newif\ifshowcomments
\showcommentstrue

\makeatletter
\let\todo\@undefined
\makeatother

\ifshowcomments
    \usepackage[textsize=scriptsize,textwidth=2.4cm]{todonotes}
\else
    \usepackage[disable,textsize=scriptsize,textwidth=2.4cm]{todonotes}
\fi

\title{Averaging Bias: Human Faithfulness Annotations are not Locally Faithful}

\author{
  Huajian Zhang\thanks{\ Equal contribution.}\qquad\qquad
  Yiyang Feng\footnotemark[1]\qquad\qquad
  Jiawei Zhou \\
  \texttt{\{huajzhang, yiyfeng, jiawei.zhou.1\}@cs.stonybrook.edu} \\
  Stony Brook University
}

\begin{document}
\maketitle

\begin{abstract}
Evaluation of faithfulness of text summarization
treats a model generated summary as faithful only if \emph{every} of its sentences is supported by the source document: a strict \textit{conjunctive rule} under which a single unsupported sentence makes the whole summary unfaithful. Yet most faithfulness benchmarks collect only one global human annotation label per summary.
We ask whether such global human labels actually implement the conjunctive rule. We hypothesize that 
annotators may accept a summary as faithful when \textit{most} sentences are faithful, not only when \textit{all} are faithful.
To test our hypothesis, we use five large language model (LLM) judges as per-sentence raters across four widely used faithfulness benchmarks. We find that global human labels correlate better with the average of per-sentence LLM judgments than with the implementation of the strict conjunctive rule.
A manual review confirms that a substantial fraction of summaries labeled faithful by humans contain genuine local factual errors. We call this tendency \emph{Averaging Bias}.
Our results reveal that human labels on widely used faithfulness benchmarks contain measurable Averaging Bias, calling for carefully structured designs for trustworthy human annotations.\footnote{Code and data available at \url{https://github.com/Zesearch/human-annotation-averging-bias}.
}

\end{abstract}

\section{Introduction}

\begin{figure}[t]
  \centering
  \includegraphics[width=\columnwidth]{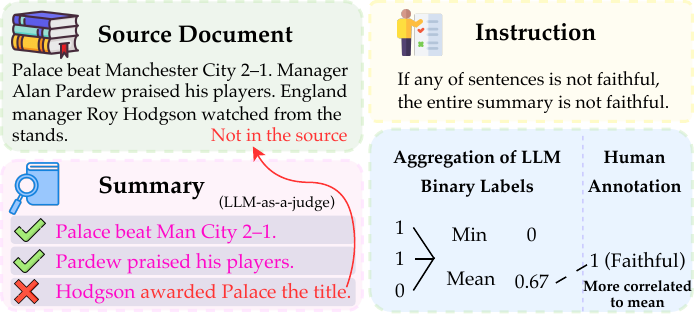}
  \caption{A 3-sentence summary with one local error.
  Faithfulness is labeled as a binary value (1 being faithful and 0 being not). Fine-grained sentence-wise LLM judge labels are aggregated by $\min$, which gives 0, and by $\mathrm{mean}$, which gives 0.67. The global human annotation label often gives 1 (\emph{Faithful}) based on the overall impression on the whole summary.
  This systematic gap is
  \emph{Averaging Bias}: it indicates that human faithfulness labels may
  not follow the strict conjunctive definition.}
  \label{fig:motivation}
\end{figure}


Faithfulness is central to evaluating document summarization: a generated summary should be factually supported by its source.
Prior work has built many benchmarks with binary faithfulness labels per
summary
\cite{wang-etal-2020-qags,fabbri-etal-2021-summeval,cao-wang-2021-cliff,goyal-durrett-2021-annotating,pagnoni-etal-2021-understanding,summac_2022,honovich-etal-2022-true-evaluating,tang-etal-2023-understanding,zhang-etal-2024-fine}.
The assumption behind them is
\emph{conjunctive}:
if each local unit (e.g., sentence) is labeled $1$ (faithful) or $0$ (unfaithful),
the global label
is the \emph{minimum} of local labels, so a single unsupported unit makes the whole generation unfaithful~\citep{wang-etal-2020-qags,polytope,fabbri-etal-2021-summeval,frank,summac_2022,honovich-etal-2022-true-evaluating,factcheckgpt2023,malaviya-etal-2024-expertqa,zhang-etal-2024-fine,tang-etal-2024-tofueval,reveal2024,STORYSUMM,FaithBench,Positional_Bias,ContextCheck,ReFACT}.



Since most benchmarks provide only a \emph{single global human annotation label} per summary~\citep{summac_2022,Protocols,honovich-etal-2022-true-evaluating,li-etal-2023-halueval,zhang-etal-2024-fine,llmaggrefact_dataset},
this implicitly assumes humans applied the conjunctive rule, but whether
they do has not been tested directly. 
We hypothesize that annotators often \textbf{do not} follow the conjunctive rule. Strict faithfulness annotation is cognitively demanding: an annotator must decompose a long generation and verify each local unit. But under load, humans may rely on heuristics rather than exhaustive verification to form a holistic impression, so the summaries that are only seemingly faithful still receive positive labels despite local errors~\cite{kahneman1973attention,tversky1974judgment,sweller1988cognitive,evans2008dual}.

To test our hypothesis, we use five open-weight LLM judges to produce per-sentence binary faithfulness judgments across four benchmarks. 
Empirically, we find that mean aggregation of per-sentence labels correlates more strongly with global human labels than minimum aggregation. 
Human validation supports the reliability of all five judges,
while a manual review also confirms that a substantial fraction of summaries labeled faithful by humans contain genuine local factual errors. We call this tendency \textbf{Averaging Bias}, i.e., human global labels behave more like the \emph{mean} of local judgments than the strict \emph{minimum} (\Figure~\ref{fig:motivation}). 
We further discuss the potential implications of this human annotation bias, including the unreliability of using global human labels for measuring faithfulness evaluators and for training models. 
\section{Related Work}

Many datasets collect human faithfulness labels for text summaries, including QAGS
\citep{wang-etal-2020-qags}, SummEval \citep{fabbri-etal-2021-summeval},
FRANK \citep{pagnoni-etal-2021-understanding}, CLIFF
\citep{cao-wang-2021-cliff}, and fine-grained span annotations
\citep{goyal-durrett-2021-annotating}. Meta-evaluation benchmarks such as
TRUE, AggreFact, DiverSumm, HaluEval, and LLM-AggreFact standardize or
aggregate these resources for evaluating automatic faithfulness scorers
\citep{honovich-etal-2022-true-evaluating,tang-etal-2023-understanding,li-etal-2023-halueval,zhang-etal-2024-fine,tang-etal-2024-minicheck}.
Each dataset's original annotation scheme and benchmark-level binary
conversion is in
\Table~\ref{tab:dataset_annotation_1} and \Table~\ref{tab:dataset_annotation_2}. These conversions collapse different annotations into a single summary-level binary label and apply a strict conjunctive rule.

A separate line of work evaluates faithfulness by decomposing a generation
into smaller units and checking each against the source: QAGS
\citep{wang-etal-2020-qags} compares answers to questions generated from
the summary and the document, NLI-style methods judge document--sentence
consistency \citep{kryscinski-etal-2020-factcc,summac_2022}, FActScore
\citep{factscore} verifies atomic facts in biographies against external
evidence, and retrieval-augmented evaluators judge individual claims before
producing an overall score \citep{RAGAs,Long_form}. 
These approaches establish a useful local view of faithfulness, where a generation is decomposed into smaller units and each unit is checked against the source. 
We build on this idea, but our goal is to use sentence-level judgments as local evidence to study how human global labels behave. We test whether they follow the strict conjunctive rule implied by faithfulness definitions.
\section{Are Global Human Faithfulness Labels Truly Faithful?}
\label{sec:method}

\subsection{Problem Setup}
\label{sec:method:setup}

We are given a source document, a generated summary, and a single human global faithfulness label, but no per-sentence human annotation. 
Since benchmarks did not contain sentence-level labels,
we split the summary into sentences and obtain a per-sentence faithfulness score from an LLM judge: either a \emph{hard} label in $\{0,1\}$ or a \emph{soft} confidence in $[0,1]$, with $1$ meaning faithful.

We then compare two ways of turning these per-sentence scores into a summary-level prediction. \emph{Minimum aggregation} implements the strict conjunctive rule: the summary is faithful only if \emph{every} sentence is (a single $0$ forces the minimum to $0$). \emph{Mean aggregation} gives partial credit, calling the summary faithful when \emph{most} sentences are (mostly $1$s with a few $0$s still average high). Our central question is whether the human global label behaves more like the minimum or more like the mean.


\subsection{A Correlation-Based Diagnostic}
\label{sec:method:diagnostic}

Our diagnostic compares the Pearson correlation between the human global label and each of the two aggregations. Assume the LLM judge's per-sentence labels are close to ground truth, with per-sentence error rate $\epsilon$ below a dataset-dependent threshold $\epsilon_0$. Under this assumption, we prove that if the human global label's Pearson correlation with mean aggregation is no less than its correlation with minimum aggregation, then human annotators are biased toward averaging (\Appendix~\ref{app:proof}). We call this bias \emph{Averaging Bias}.

The diagnostic therefore needs two ingredients: (i) the two correlations, computed from human labels and LLM judge scores, and (ii) a check that the LLM judge's per-sentence error rate $\epsilon$ falls below dataset-dependent threshold $\epsilon_0$. We validate (ii) through manual sentence-level annotation.

\subsection{Manual Review Protocol} 
\label{sec:method:protocol}

For the second ingredient of our diagnostic,
we conduct a manual review to estimate
the reliability of the LLM judges at the local level. This validation is necessary because the observed correlation pattern could otherwise have a simpler explanation: 
minimum aggregation is more sensitive to local judge
errors than mean aggregation. 
A spurious local unfaithfulness prediction can
disproportionately affect the minimum aggregation and reduce its correlation with the human global label. 
Thus, before interpreting mean $\geq$ min as evidence of Averaging Bias, we need to check that the LLM judges provide sufficiently reliable sentence-level judgments.
For these goals, we conduct two distinct manual reviews: \emph{(i)} a judge-reliability annotation that estimates each judge's
sentence-level error rate, and \emph{(ii)} a suspicious-example audit that
tests whether consensus disagreements correspond to genuine local factual
errors. We describe both studies in
\Section~\ref{sec:setup}; complete instructions and statistics are
provided in \Appendix~\ref{app:annotation}.

\section{Experimental Setup}
\label{sec:setup}

\paragraph{Judges.} Five open-weight instruction models span families and sizes: Qwen3-4B/8B/32B non-thinking \cite{qwen3}, Llama-3.1-8B-Instruct \cite{grattafiori2024llama}, and OLMo-3-7B \cite{olmo2025olmo3}. Each is prompted to output \texttt{\{"faithful": 0/1\}} per sentence (\Appendix~\ref{tab:sentence-faithfulness-prompt}); the hard label is the parsed digit, the soft label is the normalized probability of generating \texttt{1} versus \texttt{0} at the value position.

\paragraph{Datasets.} Four faithfulness benchmarks with global human labels (\Table~\ref{tab:data}): AggreFact-CNN and AggreFact-Other (CNN/DM and non-CNN portions of LLM-AggreFact, the latter filtered to $\ge2$ sentences) \citep{tang-etal-2023-understanding,tang-etal-2024-minicheck}, DiverSumm \citep{zhang-etal-2024-fine}, and the summarization portion of HaluEval \cite{li-etal-2023-halueval}. Single-sentence examples, on which all aggregators coincide, are excluded; per-subset statistics are in \Table~\ref{tab:dataset-subsets}.

\begin{table}[t]
\centering
\resizebox{\columnwidth}{!}{%
\begin{tabular}{lrc}
\hline
Benchmark & Examples & Sents/summary ($\mu\pm\sigma$) \\
\hline
AggreFact-CNN         & 1{,}017  & 3.26 $\pm$ 0.83 \\
AggreFact-Other & 644      & 2.14 $\pm$ 0.67 \\
DiverSumm             & 563      & 8.48 $\pm$ 5.17 \\
HaluEval (summ.)      & 20{,}000 & 3.59 $\pm$ 1.31 \\
\hline
\end{tabular}%
}
\caption{Benchmarks used. Each carries a binary global human faithfulness
label. The last column reports the mean and standard deviation of the
number of sentences per summary, obtained by spaCy sentence splitting over
the entire benchmark.}
\label{tab:data}
\end{table}

\paragraph{Metrics.} We report Pearson correlation $\rho$ between each
aggregated score and the human global label, with 95\% bootstrap confidence
intervals. We focus on Pearson;
Spearman, Kendall, and ROC-AUC are in
\Appendix~\ref{app:full}.

\section{LLM Judge Reliability}
\label{sec:reliability}

\paragraph{Human inspection setup}
We randomly sample 25 summaries from each benchmark, totaling 100 unique
summaries and 492 sentences. For every sampled summary, the annotator labels each sentence as \emph{Faithful}, \emph{Unfaithful}, or \emph{Unclear/fragment} according to whether it is supported by the source document. Model predictions and benchmark-level labels are hidden throughout the annotation. The result is in Figure~\ref{fig:eps-vs-eps0}.

\paragraph{Validating LLM judges}
The diagnostic is valid when the judge's per-sentence error rate $\epsilon$ falls below a benchmark-specific threshold $\epsilon_0$. Therefore, we estimate $\hat\epsilon$ (judge-annotator disagreement) from manual sentence annotations and compute $\epsilon_0$ in closed form (\Appendix~\ref{app:eps0-math},~\ref{app:annotation}). On AggreFact-CNN and AggreFact-Other, $\epsilon_0 = 0.175$ and $0.322$; on DiverSumm and HaluEval the diagnostic is \emph{unconditional}, i.e., valid for any $\epsilon \in (0, \tfrac{1}{2})$. \Figure~\ref{fig:eps-vs-eps0} shows every $\hat\epsilon$ point estimate and its 95\% bootstrap CI strictly below $\epsilon_0$: Qwen3-32B is most accurate ($\hat\epsilon \in [0.100, 0.143]$), while the most error-prone judge (Llama-3.1-8B, up to $0.286$ on AggreFact-Other) still clears its threshold. The conclusions in \Table~\ref{tab:meta-eval-main} therefore hold throughout.

\section{Results}
\label{sec:results}

\begin{table*}[t]
\centering\small
\setlength{\tabcolsep}{5pt}
\begin{tabular}{l cc cc cc cc}
\toprule
Model & \multicolumn{2}{c}{\textbf{AggreFact-CNN}} & \multicolumn{2}{c}{\textbf{AggreFact-Other}} & \multicolumn{2}{c}{\textbf{DiverSumm}} & \multicolumn{2}{c}{\textbf{HaluEval}} \\
\cmidrule(lr){2-3} \cmidrule(lr){4-5} \cmidrule(lr){6-7} \cmidrule(lr){8-9}
 & $\min$ & $\mathrm{mean}$ & $\min$ & $\mathrm{mean}$ & $\min$ & $\mathrm{mean}$ & $\min$ & $\mathrm{mean}$ \\
\midrule
Qwen3-4B & $0.22_{\pm 0.10}$ & $\mathbf{0.24}_{\pm 0.09}$ & $0.07_{\pm 0.07}$ & $\mathbf{0.18}_{\pm 0.08}$ & $\mathbf{0.46}_{\pm 0.08}$ & $0.41_{\pm 0.06}$ & $0.26_{\pm 0.01}$ & $\mathbf{0.31}_{\pm 0.01}$ \\
Qwen3-8B & $0.27_{\pm 0.09}$ & $0.27_{\pm 0.09}$ & $0.11_{\pm 0.07}$ & $\mathbf{0.21}_{\pm 0.08}$ & $0.38_{\pm 0.09}$ & $\mathbf{0.41}_{\pm 0.06}$ & $0.31_{\pm 0.01}$ & $\mathbf{0.40}_{\pm 0.01}$ \\
Qwen3-32B & $0.39_{\pm 0.09}$ & $\mathbf{0.41}_{\pm 0.08}$ & $0.19_{\pm 0.06}$ & $\mathbf{0.24}_{\pm 0.07}$ & $0.40_{\pm 0.09}$ & $\mathbf{0.41}_{\pm 0.06}$ & $0.39_{\pm 0.01}$ & $\mathbf{0.47}_{\pm 0.01}$ \\
Llama-3.1-8B & $\mathbf{0.25}_{\pm 0.08}$ & $0.23_{\pm 0.09}$ & $0.14_{\pm 0.07}$ & $\mathbf{0.17}_{\pm 0.08}$ & $\mathbf{0.33}_{\pm 0.07}$ & $0.28_{\pm 0.06}$ & $0.31_{\pm 0.01}$ & $\mathbf{0.40}_{\pm 0.01}$ \\
OLMo-3-7B & $0.20_{\pm 0.07}$ & $\mathbf{0.21}_{\pm 0.08}$ & $0.15_{\pm 0.06}$ & $\mathbf{0.18}_{\pm 0.07}$ & $0.15_{\pm 0.09}$ & $\mathbf{0.29}_{\pm 0.08}$ & $0.14_{\pm 0.01}$ & $\mathbf{0.26}_{\pm 0.01}$ \\
\bottomrule
\end{tabular}
\caption{Pearson correlation (hard predictions) between aggregated sentence-level judgments and the human global label, per (judge, benchmark). Cells show $\rho \pm \delta$ with $\delta$ the half-width of the 95\% bootstrap CI; the higher of $\min$ and $\mathrm{mean}$ per pair is in bold. $\mathrm{mean}$ exceeds $\min$ for all five judges on AggreFact-Other and HaluEval (Averaging Bias by \Section~\ref{sec:method:diagnostic}); on AggreFact-CNN and DiverSumm $\mathrm{mean}$ wins for three of five judges.}
\label{tab:meta-eval-main}
\end{table*}

\begin{figure}[t]
  \centering
  \includegraphics[width=\columnwidth]{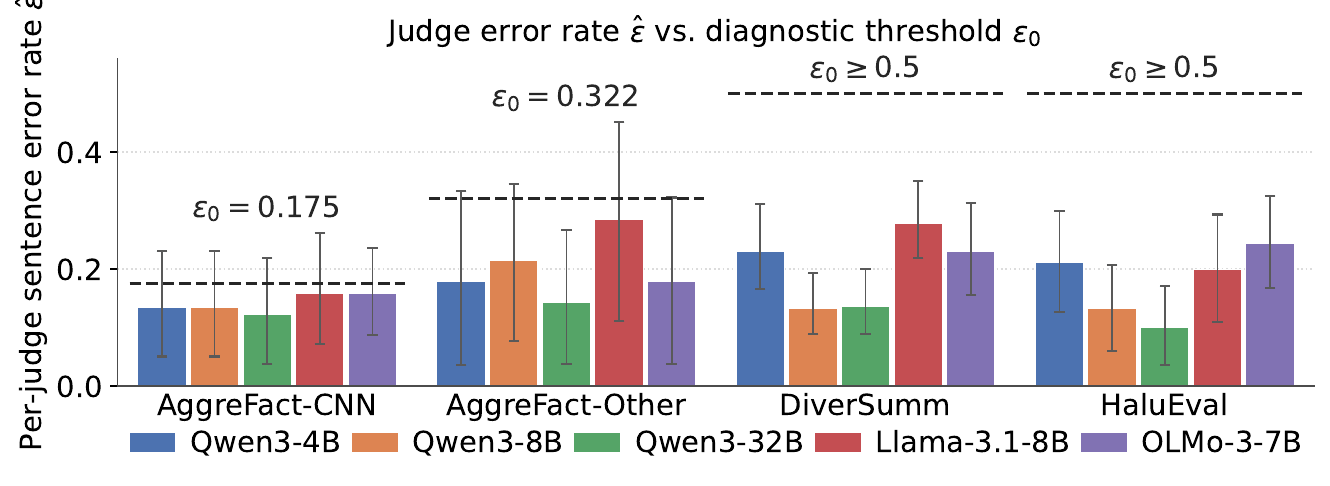}
  \caption{Per-judge sentence-level error rate $\hat\epsilon$ (bars; 95\% bootstrap CI) versus the \Section~\ref{sec:method:diagnostic} threshold $\epsilon_0$ (dashed), from 25 randomly sampled summaries per benchmark. Every $\hat\epsilon$ satisfies $\hat\epsilon<\epsilon_0$.\protect\footnotemark{} Per-(judge, benchmark) CIs in \Table~\ref{tab:eps-long}.}
  \label{fig:eps-vs-eps0}
\end{figure}
\footnotetext{On DiverSumm and HaluEval the noise condition holds \emph{unconditionally}: $\rho(Y,A)>\rho(Y,B)$ for all $\epsilon\in(0,\tfrac12)$, so no finite threshold exists in that range.}

\paragraph{Mean aggregation matches human labels at least as well as minimum.}
\Table~\ref{tab:meta-eval-main} reports the main comparison. On AggreFact-Other and HaluEval, $\mathrm{mean}$ exceeds $\min$ for \emph{all five} judges, with large gaps of up to $0.11$ on AggreFact-Other and $0.12$ on HaluEval. Besides, the HaluEval bootstrap CIs ($\pm 0.01$) make the per-judge gaps statistically clear. Both benchmarks therefore give uniform evidence of Averaging Bias. The signal is partial on AggreFact-CNN and DiverSumm, where $\mathrm{mean}$ wins for three of five judges. Across all four benchmarks, Qwen3-32B (the largest judge) and OLMo-3-7B show $\mathrm{mean} \ge \min$ uniformly, while the smaller Qwen3 models and Llama-3.1-8B occasionally flip on the benchmarks where the signal is weaker. More results can be found in the \Appendix~\ref{app:full}.

\paragraph{Averaging Bias can be significant even with small correlation gaps.}
\Table~\ref{tab:meta-eval-main} indicates strong Averaging Bias on AggreFact-Other and HaluEval, but the gaps between the $\mathrm{mean}$ and $\min$ correlations are small on AggreFact-CNN and DiverSumm, where. We conduct a simulation showing that Averaging Bias can still be significant with a small gap. Specifically, we sample sentence-level faithfulness labels according to the ratio of faithful sentences in our manual annotation, and simulate how humans aggregate them into a summary-level label. A fraction $y$ of summaries applies the averaging rule (faithful when at least half of the sentences are faithful), and the remaining $1-y$ apply the minimum rule (faithful only when every sentence is faithful). We then compute the $\mathrm{mean}$ and $\min$ correlations against these simulated labels, for a perfect judge and for a judge with the measured error rate $\bar\epsilon$ (the judge average of $\hat\epsilon$). \Figure~\ref{fig:oracle-mixture} shows that a small gap can hide substantial bias even under the perfect judge. On AggreFact-CNN, the gap is only $+0.03$ when $60\%$ of the labels are averaging-biased, and on DiverSumm it is $0.00$ when $40\%$ are. Therefore, the small gaps for these two benchmarks in \Table~\ref{tab:meta-eval-main} do not certify unbiased labels.

\begin{figure}[t]
  \centering
  \includegraphics[width=0.98\columnwidth]{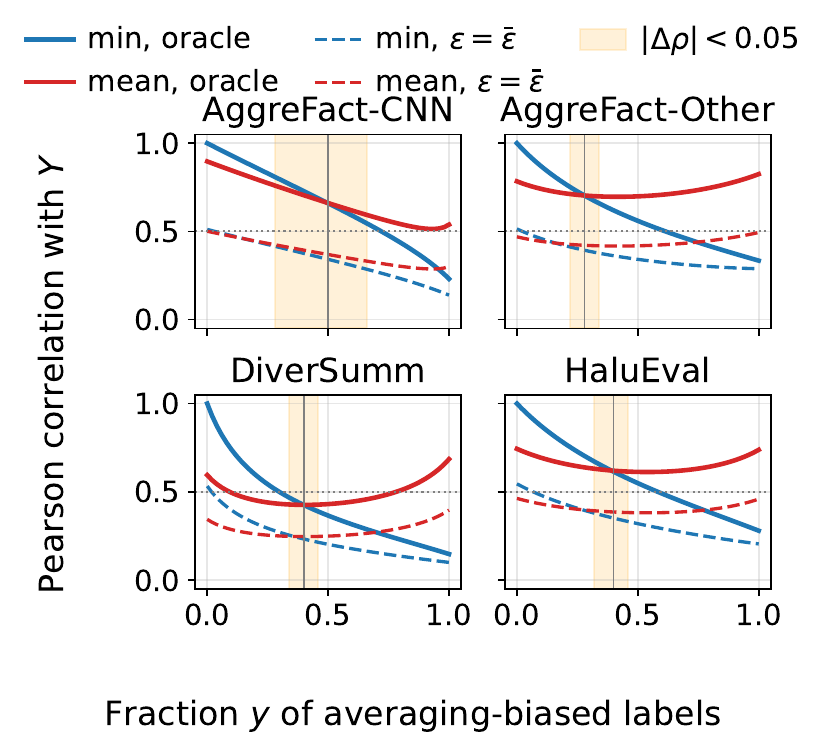}
  \caption{Pearson correlation between the simulated global label and each aggregation, as the fraction $y$ of averaging-biased labels varies. Solid curves use a perfect judge and dashed curves use the measured judge error $\bar\epsilon$. The vertical line marks where the two solid curves cross, and the shaded band marks small correlation gaps $|\Delta\rho|<0.05$ for the perfect judge.}
  \label{fig:oracle-mixture}
\end{figure}


\paragraph{Manual review confirms genuine human-label errors.}
We define a summary as \emph{suspicious} when its human label is faithful but all five LLM judges flag at least one sentence as unfaithful. For each benchmark, an annotator reviews up to 25 such summaries and re-labels each flagged sentence against the source document, using the same \emph{Faithful}, \emph{Unfaithful}, and \emph{Unclear} options. \Figure~\ref{fig:pies} reports both granularities. The fraction of human-faithful summaries that are \emph{suspicious} (left) is 2\%, 40\%, 7\% and 14\% on AggreFact-CNN, AggreFact-Other, DiverSumm and HaluEval. This mirrors \Table~\ref{tab:meta-eval-main}: the two benchmarks with the weaker $\mathrm{mean} \ge \min$ signal (AggreFact-CNN and DiverSumm) are also the ones with the fewest suspicious summaries, showing that the dataset-level diagnostic reflects how much bias is actually present. On up to 25 such summaries per benchmark, manual sentence-level review (right) confirms 77\% (10/13), 79\% (22/28), 50\% (5/10) and 89\% (24/27) of the flagged sentences as genuinely \emph{unfaithful}. These high rates ($\ge 50\%$) show that the per-sentence flags are themselves valid, so the diagnostic is also useful at the example level for locating specific mislabeled summaries.

\paragraph{Error types.} 
We provide representative error cases per benchmark in \Table~\ref{tab:manual-error-cases}.
Confirmed errors recur in several forms, including unsupported or hallucinated content, entity/attribute substitution, relation or event distortion, overstatement, and attribution errors. Boilerplate or navigation leakage is also common, especially in HaluEval, whose CNN/DailyMail references are known to contain such artefacts \citep{fabbri-etal-2021-summeval,guo-etal-2022-questioning}.
Note that fragments from sentence splitting are not counted as label errors.

\paragraph{Extension to subclaim granularity.}
Our primary analysis treats sentences as local units. To examine whether the
same pattern occurs within sentences, we repeat the analysis using human
sentence-level labels and LLM judgments over atomic subclaims on TofuEval and
MSumBench. The results shows mean aggregation exceeds minimum aggregation for every judge--dataset pair. A manual audit also confirms that some human-faithful sentences contain unsupported subclaims. Full setup and results are provided
in \Appendix~\ref{app:subclaim}.

\begin{figure}[t]
  \centering
  \includegraphics[width=0.995\columnwidth]{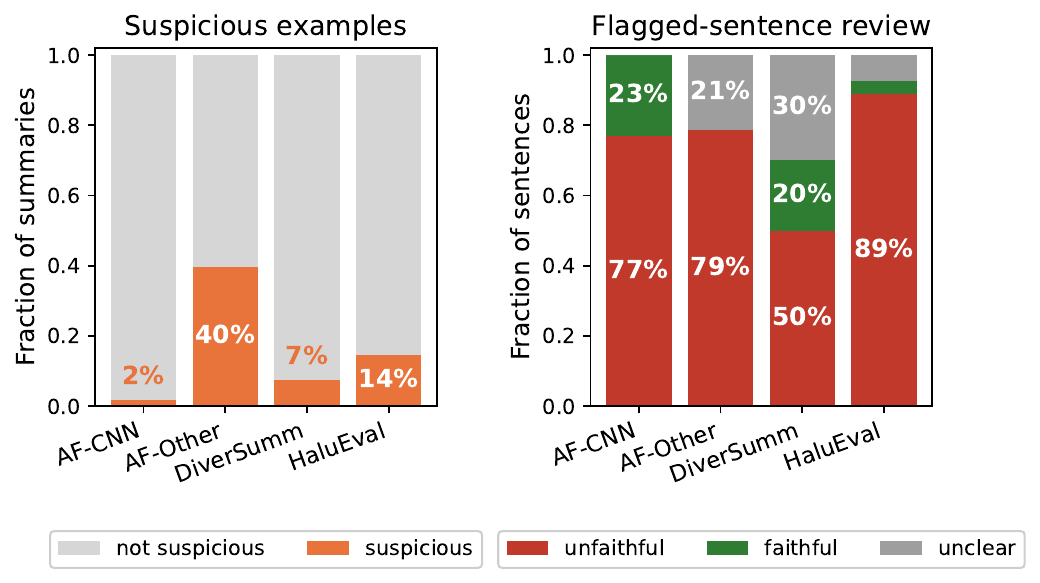}
  \caption{(Left, example-level) fraction of human-faithful summaries that are \emph{suspicious}: all five LLM judges flag at least one sentence. (Right, sentence-level) manual verdicts on the all-five-flagged sentences inside up to 25 suspicious summaries per benchmark: \emph{unfaithful} confirms an averaging-bias error, \emph{faithful} means the judges were too strict, \emph{unclear} otherwise.}
  \label{fig:pies}
\end{figure}

\section{Discussion}
\label{sec:discussion}

\paragraph{Higher agreement with human labels $\ne$ better faithfulness evaluator.}
Faithfulness evaluators are routinely ranked by agreement with the human global label. Our results suggest this conflates two qualities: when the labels themselves average over local judgments rather than apply the strict rule, a metric also tops the ranking by averaging, not by being a stricter checker. The agreement number measures alignment with human behavior, not adherence to the conjunctive definition. 


\paragraph{Implications for training.} Many training pipelines
use human-labeled data as
supervision; if global labels reward mostly-correct long responses despite
local errors, models trained on them may inherit the bias. 

\section{Conclusion}



We introduce \emph{Averaging Bias}: human global faithfulness labels behave more like the \emph{mean} of per-sentence judgments than \emph{minimum}. We prove a correlation-based diagnostic and apply it across four benchmarks and five LLM judges; $\mathrm{mean}$ exceeds $\min$ for all five judges on two benchmarks, and for three of five judges on the other two benchmarks.
Manual review confirms genuine local errors hidden by positive human labels. 


\section*{Limitations}

Our diagnostic uses LLM judges as a proxy for local human annotation, which we do not have at scale. Stronger proprietary models or larger open-weight models may provide more reliable sentence-level judgments.
The manual review covers 20 cases per benchmark and is a diagnostic rather than the primary evidence. 
Sentence splitting can produce malformed units (the ``fragment'' category), and results depend on the judge prompt and on binary label normalization across benchmarks with different annotation procedures.

We use sentences as natural local semantic units. This choice is simple
and broadly applicable, but it is not the finest possible granularity. Many
faithfulness evaluators decompose generations into smaller units, such as subclaims, atomic facts, or fine-grained factual propositions
\citep{alignscore,factscore,tang-etal-2024-minicheck,mishra2024finegrained,VeriScore,DiscInfer,Discourse,factcg}. Recent work on long-document factual
consistency evaluation further highlights the importance of fine-grained evaluation, showing that continuous local assessments can be more informative than summary-level binary labels \citep{Tale}.
Using such finer-grained units could reveal local errors that are hidden inside
a single sentence and may provide a more precise test of local-to-global
faithfulness aggregation.

\bibliography{references}

\appendix

\section{Details of Correlation-Based Diagnostic}
\label{app:proof}

\subsection{Setup}

Fix integers $M \geq 2$ and let $p \in (0,1)$. Let $S_1, \ldots, S_M$ be
i.i.d.\ Bernoulli$(p)$ random variables (the latent sentence-level faithfulness
labels). Let $\hat S_1, \ldots, \hat S_M \in [0,1]$ be the model's per-sentence
predictions, conditionally independent given $S$, with the conditional
distribution of $\hat S_i$ depending only on $S_i$. We treat two prediction
models in parallel.

\paragraph{Hard predictions.}
$\hat S_i \in \{0,1\}$ with $\hat S_i = S_i \oplus E_i$, where
$E_1, \ldots, E_M$ are i.i.d.\ Bernoulli$(\epsilon)$ for $\epsilon \in [0,1/2)$,
independent of $S$, and $\oplus$ is XOR. Equivalently,
$\Prob(\hat S_i = 1 \mid S_i = 1) = 1 - \epsilon$ and
$\Prob(\hat S_i = 1 \mid S_i = 0) = \epsilon$.

\paragraph{Soft predictions.}
$\hat S_i \in [0,1]$ derived from per-sentence logits as
$\hat S_i = \exp(\ell_{i,1}) / (\exp(\ell_{i,0}) + \exp(\ell_{i,1}))$, with
conditional first moments
\[
\begin{aligned}
  \E[\hat S_i \mid S_i = 1] &= 1 - \epsilon, \\
  \E[\hat S_i \mid S_i = 0] &= \epsilon, \\
  \epsilon &\in [0,1/2).
\end{aligned}
\]
We further assume the regularity condition that the conditional second moments
$\E[\hat S_i^2 \mid S_i]$ are continuous in $\epsilon$ at $\epsilon = 0$, which
is equivalent to $\hat S_i \to S_i$ in $L^2$ as $\epsilon \to 0$ and holds for
any softmax-of-logits family in which the logit gap concentrates on the correct
sign as $\epsilon \to 0$. Note that at $\epsilon = 0$ we have
$\E[\hat S_i \mid S_i = 1] = 1$ and $\E[\hat S_i \mid S_i = 0] = 0$, so since
$\hat S_i \in [0,1]$ we automatically have $\hat S_i = S_i$ almost surely; the
soft model collapses to the noiseless hard model in this limit.

\paragraph{Aggregators and labels.}
Set
\[
  K = \sum_{i=1}^M S_i, \qquad \hat K = \sum_{i=1}^M \hat S_i.
\]
The two aggregators we compare are
\[
  A = \min_i \hat S_i, \qquad B = \frac{1}{M}\sum_{i=1}^M \hat S_i = \frac{\hat K}{M}.
\]
In the hard case, $A = \1[\hat K = M]$, but this identity does not hold in the
soft case in general. The human label $Y = h(K)$ is determined by a function
$h: \{0,\ldots,M\} \to \{0,1\}$. We treat two cases:
\begin{itemize}
  \item \emph{Unbiased rater (Claim 1):} $h(K) = \1[K = M]$.
  \item \emph{Biased rater (Claim 2):} $h(K) = \1[K \geq j_0]$ for an integer
  $j_0$ with $1 \leq j_0 \leq M-1$.
\end{itemize}
For a random variable $T$, write
$\rho(Y, T) := \Cov(Y, T) / \sqrt{\Var(Y)\Var(T)}$ for Pearson's correlation.

\subsection{Reduction to Functions of $K$}

\begin{lemma}\label{lem:cond-exp}
For each $k \in \{0,\ldots,M\}$:
\begin{enumerate}
  \item In both prediction models,
  \[
    \E[B \mid K = k] = \epsilon + \frac{1-2\epsilon}{M}\, k.
  \]
  \item In the hard prediction model,
  \[
    \E[A \mid K = k] = (1-\epsilon)^k \epsilon^{M-k}.
  \]
  \item In the soft prediction model, with
  $G_s(t) := \Prob(\hat S_1 > t \mid S_1 = s)$,
  \[
    \E[A \mid K = k] = \int_0^1 G_1(t)^k\, G_0(t)^{M-k}\, dt.
  \]
\end{enumerate}
In all cases, $\E[A \mid S]$ and $\E[B \mid S]$ depend on $S$ only through $K$.
\end{lemma}

\begin{proof}
For $B$, by linearity and the conditional means,
\[
  \E[\hat S_i \mid S_i] = (1-\epsilon)S_i + \epsilon(1-S_i) = \epsilon + (1-2\epsilon)S_i,
\]
which holds in both prediction models. Hence
$\E[\hat K \mid K] = \epsilon M + (1-2\epsilon)K$ and dividing by $M$ gives (1).

For $A$ in the hard case: conditional on $S$ the variables $\hat S_i$ are
independent with $\Prob(\hat S_i = 1 \mid S_i = 1) = 1-\epsilon$ and
$\Prob(\hat S_i = 1 \mid S_i = 0) = \epsilon$. Hence
\[
  \E[A \mid S] = \prod_{i=1}^M \Prob(\hat S_i = 1 \mid S_i),
\]
which equals $(1-\epsilon)^k \epsilon^{M-k}$ when $\sum S_i = k$, giving (2).

For $A$ in the soft case: by conditional independence of $\hat S_i$ given $S$,
\[
\begin{aligned}
  \Prob(A > t \mid S) &= \prod_{i=1}^M \Prob(\hat S_i > t \mid S_i) \\
  &= G_1(t)^k\, G_0(t)^{M-k}
\end{aligned}
\]
when $\sum S_i = k$. Since $\hat S_i \in [0,1]$,
$\E[A \mid S] = \int_0^1 \Prob(A > t \mid S)\, dt$, which gives (3). In all
cases the resulting expression is symmetric in $S$ and depends only on $K$.
\end{proof}

\begin{lemma}\label{lem:tower}
For any function $T$ of $\hat S$, with $\tilde T(K) := \E[T \mid K]$,
\[
  \Cov(Y, T) = \Cov\!\bigl(h(K),\, \tilde T(K)\bigr).
\]
\end{lemma}

\begin{proof}
Because $Y = h(S)$ is determined by $S$, and $h$ is symmetric so $h(S) = h(K)$,
conditioning on $S$ gives
\[
\begin{aligned}
  \E[Y T] &= \E\!\bigl[\E[Y T \mid S]\bigr] \\
  &= \E\!\bigl[Y\, \E[T \mid S]\bigr]
   = \E\!\bigl[h(K)\, \tilde T(K)\bigr],
\end{aligned}
\]
where the last equality uses that $\E[T \mid S]$ depends on $S$ only through
$K$ (\Lemma~\ref{lem:cond-exp}). Similarly $\E[T] = \E[\tilde T(K)]$. Subtracting
the products of expectations gives the claimed identity.
\end{proof}

\subsection{Claim of Unbiased Rater Under Pearson Correlation}

\begin{theorem}\label{thm:claim1}
At $\epsilon = 0$, with $h(K) = \1[K = M]$, in either the hard or the soft
prediction model,
\[
  \rho(Y, A) = 1 \;>\; \rho(Y, B) > 0.
\]
\end{theorem}

\begin{proof}
At $\epsilon = 0$, both models give $\hat S_i = S_i$ almost surely (for the soft
model this follows from $\hat S_i \in [0,1]$ together with the conditional
means $\E[\hat S_i \mid S_i] \in \{0,1\}$). Hence $\hat K = K$ and
\[
  A = \min_i \hat S_i = \min_i S_i = \1[K = M] = Y,
\]
so $\rho(Y, A) = 1$.

For $B = K/M$ we compute. Writing $\beta_1 := \Prob(K=M) = p^M$,
\begin{gather*}
  \E[Y] = \beta_1, \qquad \Var(Y) = \beta_1(1-\beta_1), \\
  \E[B] = p, \qquad \Var(B) = \tfrac{p(1-p)}{M}, \\
  \E[YB] = \E\!\bigl[\1[K=M]\,\tfrac{K}{M}\bigr] = \Prob(K=M) = \beta_1.
\end{gather*}
Hence $\Cov(Y, B) = \beta_1 - \beta_1 p = \beta_1(1-p) > 0$ and
\[
\begin{aligned}
  \rho(Y, B)^2 &= \frac{\beta_1^2(1-p)^2}{\beta_1(1-\beta_1) \cdot p(1-p)/M} \\
  &= \frac{M p^{M-1}(1-p)}{1 - p^M}.
\end{aligned}
\]
Using $1 - p^M = (1-p)(1 + p + \cdots + p^{M-1})$,
\[
  \rho(Y, B)^2 = \frac{M p^{M-1}}{1 + p + p^2 + \cdots + p^{M-1}}.
\]
For $p \in (0,1)$, each term $p^i$ in the denominator satisfies
$p^i \geq p^{M-1}$, with strict inequality whenever $i < M-1$. Therefore the
denominator strictly exceeds $M p^{M-1}$, giving $\rho(Y, B)^2 < 1$. Combined
with $\rho(Y, B) > 0$, this yields $0 < \rho(Y, B) < 1 = \rho(Y, A)$.
\end{proof}

\begin{corollary}\label{cor:claim1-eps}
For each $M \geq 2$ and $p \in (0,1)$ there exists
$\epsilon_0 = \epsilon_0(M, p) > 0$ such that $\rho(Y, A) > \rho(Y, B)$ for all
$\epsilon \in [0, \epsilon_0)$, in either the hard or the soft prediction
model.
\end{corollary}

\begin{proof}
We argue by continuity at $\epsilon = 0$. We must show that
$\E[Y], \E[A], \E[B], \E[YA], \E[YB], \E[A^2], \E[B^2]$ are continuous in
$\epsilon$ at $0$, and that both $\Var(A)$ and $\Var(B)$ are strictly positive
at $\epsilon = 0$.

\emph{Hard case.} Each of the listed quantities is a polynomial in $\epsilon$
(the conditional probabilities $\Prob(\hat S_i = 1 \mid S_i)$ are linear in
$\epsilon$, and the listed expectations are finite sums of products of such
terms), hence continuous on $[0,1/2)$.

\emph{Soft case.} For $B$:
$\E[B]$ is linear in $\epsilon$ by \Lemma~\ref{lem:cond-exp}, and
$\E[YB] = \E[h(K)\,\tilde B(K)]$ with $\tilde B(K) = \epsilon + (1-2\epsilon)K/M$
by \Lemma~\ref{lem:tower}, which is linear in $\epsilon$. For $\E[B^2]$,
\[
  \E[B^2] = \frac{1}{M^2}\!\left(\sum_i \E[\hat S_i^2]
  + \sum_{i \neq j} \E[\hat S_i \hat S_j]\right),
\]
where conditional independence gives
$\E[\hat S_i \hat S_j] = \E[(\epsilon + (1-2\epsilon)S_i)(\epsilon + (1-2\epsilon)S_j)]$
for $i \neq j$, a polynomial in $\epsilon$, while $\E[\hat S_i^2]$ is continuous
in $\epsilon$ at $0$ by the regularity assumption.

For $A$: the regularity assumption $\E[\hat S_i^2 \mid S_i]$ continuous at
$\epsilon = 0$ together with $\E[\hat S_i \mid S_i = 1] = 1 - \epsilon$ and
$\E[\hat S_i \mid S_i = 0] = \epsilon$ implies
\[
  \E[(\hat S_i - S_i)^2] = \E[\hat S_i^2] - 2\E[\hat S_i S_i] + \E[S_i^2]
  \xrightarrow{\epsilon \to 0} 0,
\]
i.e., $\hat S_i \to S_i$ in $L^2$. The map
$(x_1,\ldots,x_M) \mapsto \min_i x_i$ from $[0,1]^M$ to $[0,1]$ is $1$-Lipschitz
in the sup norm (and hence in $L^2$ on a bounded domain), so
$A = \min_i \hat S_i \to \min_i S_i$ in $L^2$ as $\epsilon \to 0$. By
$L^2$ convergence, $\E[A]$, $\E[A^2]$, and $\E[YA] = \E[Y \min_i \hat S_i]$ are
continuous at $\epsilon = 0$ (using $|Y| \leq 1$ for the last one).

\emph{Variances at $\epsilon = 0$.} At $\epsilon = 0$, $A = Y$ so
$\Var(A) = \Var(Y) = p^M(1-p^M) > 0$ (since $p \in (0,1)$ and $M \geq 2$ make
$p^M \in (0,1)$). And $\Var(B) = p(1-p)/M > 0$. Hence both correlations are
well-defined on a neighborhood of $0$.

The strict inequality $\rho(Y, A) > \rho(Y, B)$ at $\epsilon = 0$ from
\Theorem~\ref{thm:claim1} therefore extends to a neighborhood
$[0, \epsilon_0)$ by continuity.
\end{proof}

\begin{corollary}[Varying summary length]
\label{cor:varying-M}
Let the summary length $M$ be a random variable supported on a finite set
$\mathcal{M} \subseteq \{2, 3, \ldots\}$, and condition on $M$ as in the Setup,
so that $K \mid M \sim \mathrm{Binomial}(M, p)$ and the unbiased rater sets
$Y = \1[K = M]$ for every value of $M$. Then there exists $\epsilon_0' > 0$
such that $\rho(Y, A) > \rho(Y, B)$ for all $\epsilon \in [0, \epsilon_0')$,
where $\rho$ denotes Pearson correlation over the pooled population. Hence
$\rho(Y, A) \leq \rho(Y, B)$ still implies the rater is biased.
\end{corollary}

\begin{proof}
At $\epsilon = 0$ the judge is exact, so for every summary, whatever its length
$M$, $A = \min_i S_i = \1[K = M] = Y$; hence $A = Y$ almost surely over the
pooled population and $\rho(Y, A) = 1$. Since $Y = 1$ forces $B = 1$ while
$Y = 0$ admits several values of $B = K/M$, $B$ is not an affine function of
$Y$, so $\rho(Y, B) < 1$. Each pooled moment is the finite mixture
$\E[\,\cdot\,] = \sum_{m \in \mathcal{M}} \Prob(M = m)\, \E[\,\cdot \mid M = m]$
of the fixed-length moments, each continuous in $\epsilon$ at $0$ by
\Corollary~\ref{cor:claim1-eps}; a finite mixture of continuous functions is
continuous, and $\Var(A) = \Var(Y) > 0$ and $\Var(B) > 0$ at $\epsilon = 0$.
Continuity then extends $\rho(Y, A) > \rho(Y, B)$ from $\epsilon = 0$ to a
neighborhood $[0, \epsilon_0')$.
\end{proof}

\begin{corollary}\label{cor:bias-detection}
Let $M \geq 2$ and $p \in (0,1)$. Let $\epsilon_0 = \epsilon_0(M, p) > 0$ be the
threshold guaranteed by \Corollary~\ref{cor:claim1-eps}. Assume the model's
error rate satisfies $\epsilon \in [0, \epsilon_0)$, under either the hard or
the soft prediction model. If the observed Pearson correlations satisfy
$\rho(Y, A) \leq \rho(Y, B)$, then the human rater cannot be unbiased; that is,
$h(K) \neq \1[K = M]$.
\end{corollary}

\begin{proof}
By contraposition. Suppose, for contradiction, that $h(K) = \1[K = M]$. Since
$\epsilon \in [0, \epsilon_0)$, \Corollary~\ref{cor:claim1-eps} gives the strict
inequality $\rho(Y, A) > \rho(Y, B)$ in either prediction model, contradicting
the assumption $\rho(Y, A) \leq \rho(Y, B)$. Hence
$h(K) \neq \1[K = M]$.
\end{proof}

\begin{remark}[Diagnostic Test for Rater Bias]
\Corollary~\ref{cor:bias-detection} is the formal statement of the diagnostic
used as \Section~\ref{sec:method:diagnostic} in the main text. The result is robust to the
choice of LLM scoring scheme: it holds whether $\hat S_i$ is a hard binary
label or a soft softmax probability. When the LLM judge is highly accurate
($\epsilon \approx 0$), an unbiased rater should align more strongly with the
strict aggregator $A$ than with the mean aggregator $B$. While the exact
equality $\rho(Y, A) = 1$ (\Theorem~\ref{thm:claim1}) is fragile and broken by
any positive $\epsilon$, the strict inequality $\rho(Y, A) > \rho(Y, B)$ is
robust within $[0, \epsilon_0)$. Consequently, if the mean aggregator $B$
correlates with $Y$ as well as or better than the strict aggregator $A$, this
is a strong statistical signal that the raters employ a biased heuristic, the
canonical case being $h(K) = \1[K \geq j_0]$ for some $j_0 < M$.
\end{remark}

\subsection{Computing $\epsilon_0$}
\label{app:eps0-math}

We compute $\epsilon_0(D)$ for dataset $D$ from $\hat p$
(\Appendix~\ref{app:annotation}) and the dataset's empirical distribution of
summary lengths $\{M_g\}_{g=1}^{N_D}$ (read from the LLM cache: $M_g$ is the
number of sentences for summary $g$). Under the model of
\Section~\ref{sec:method:diagnostic} with hard predictions,
$S_1,\ldots,S_M\stackrel{\mathrm{iid}}{\sim}\mathrm{Bernoulli}(p)$,
$K=\sum_i S_i$; the judge outputs $\hat S_i$ are i.i.d.\
$\mathrm{Bernoulli}(r)$ with $r=p+\epsilon(1-2p)$; the unbiased rater is
$Y=\1[K=M]$; and the aggregators are $A=\min_i\hat S_i$ and $B=\hat K/M$
with $\hat K=\sum_i\hat S_i$.

\paragraph{Pooled moments for varying $M$.}
Define $\mu_q = \tfrac{1}{N}\textstyle\sum_g p^{M_g}$, $\mu_{\mathrm{inv}M} = \tfrac{1}{N}\textstyle\sum_g 1/M_g$, $\mu_A(\epsilon) = \tfrac{1}{N}\textstyle\sum_g r^{M_g}$, $\mu_{YA}(\epsilon) = \tfrac{1}{N}\textstyle\sum_g \bigl(p(1-\epsilon)\bigr)^{M_g}.$
Pooled moments are finite mixtures of the fixed-$M$ moments, giving
\begin{align*}
\rho(Y,A) &= \frac{\mu_{YA}-\mu_q\,\mu_A}
                  {\sqrt{\mu_q(1-\mu_q)\,\mu_A(1-\mu_A)}},\\[4pt]
\rho(Y,B) &= \frac{\mu_q\,(1-p)(1-2\epsilon)}
                  {\sqrt{\mu_q(1-\mu_q)\,r(1-r)\,\mu_{\mathrm{inv}M}}}.
\end{align*}
$\epsilon_0(D)$ is the smallest $\epsilon\in(0,\tfrac12)$ at which
$\rho(Y,A)=\rho(Y,B)$, located by bisection. At $\epsilon=0$ the judge is
exact, so $A=Y$ per summary for every $M$, giving $\rho(Y,A)=1>\rho(Y,B)$;
$\rho(Y,A)$ decays faster than $\rho(Y,B)$ as $\epsilon$ grows and the
first crossing is the threshold. \emph{If no root exists in $(0,\tfrac12)$},
$\rho(Y,A)>\rho(Y,B)$ for every judge error in that interval and the noise
condition is unconditional (as on DiverSumm and HaluEval in our data).

\paragraph{Estimating $\hat p$ and $\hat\epsilon$.}
From the annotated dataset $D$ (summaries $g=1,\ldots,G_D$, sentences
$j=1,\ldots,M_g$ with ground-truth labels $S_{gj}\in\{0,1\}$ and cached
judge labels $\hat S_{gj}^{J}$),
\[
\hat p_D = \frac{\sum_{g,j} S_{gj}}{\sum_g M_g},
\qquad
\hat\epsilon^{J}_D = \frac{\sum_{g,j}\1\!\bigl[\hat S_{gj}^{J}\neq S_{gj}\bigr]}{\sum_g M_g}.
\]
``Unclear / fragment'' labels are excluded from both sums. 95\% confidence
intervals resample whole summaries (sentences within a summary are
correlated): 2000 bootstrap replicates, $2.5/97.5$ percentile intervals.

\section{Responsible NLP Research} \label{sec:appendix:checklist}

\subsection{Dataset License}
Aggrefact is under Apache-2.0 license. DiverSumm and HaluEval
are under MIT license. We ensure that these benchmarks were used solely for academic purposes. For data safety, content filtering was conducted when the creators built
the original datasets. It is not avoidable that some documents especially the news, may contain public names/entities and possibly sensitive/offensive content.


\subsection{Recruitment, Consent And Ethics of Annotation}
Annotators are the authors of this paper, graduate-student NLP researchers
with native or near-native English proficiency; no external participants
were recruited and no payment was made. Ethics-board approval was not required because no external human subjects were recruited and only public benchmark data was used.

\subsection{AI Assistant Use}
AI assistants were used for writing in grammar and style checking, and coding in debugging, documentation and comments. All claims, experiments, annotations, and final text were written, implemented, and checked by authors.


\subsection{Inference Setup And Hyperparameters}
\label{app:inference}
We use 2 H100 GPUs for inference. All judges are served through a local vLLM OpenAI-compatible endpoint in \texttt{bfloat16}. Decoding is greedy: \texttt{temperature}=0.0, \texttt{top\_p}=1.0, \texttt{max\_tokens}=8192 per call. For Qwen3 we set \texttt{enable\_thinking}=false to suppress \texttt{<think>} blocks so the model emits the JSON answer directly. Soft labels read \texttt{top\_logprobs}=20 at the JSON value position. Sentence splitting uses \texttt{spaCy==3.8.13} with \texttt{en\_core\_web\_sm}. All per-judge predictions are cached on disk so $\min$ and $\mathrm{mean}$ aggregation are compared on identical inputs. The prompt is in \Table~\ref{tab:sentence-faithfulness-prompt}.

\section{Annotation Collection Details}
\label{app:annotation}

We collect two complementary sets of manual labels via the same
self-contained HTML annotation page (released with the code): one for
sentence-level ground truth used to estimate $\hat\epsilon$ and $\hat p$
(\S\ref{app:annot-sent}), and one to manually adjudicate suspicious
summaries (\S\ref{app:annot-susp}). The pages show the source document
on the left and the summary on the right; model predictions and the global
human label are \emph{hidden} during annotation to avoid anchoring.
Three options appear wherever a label is required:
\textbf{Faithful}, \textbf{Unfaithful}, and \textbf{Unclear / fragment}.


\subsection{Sentence-Level Ground Truth ($\hat\epsilon$, $\hat p$)}
\label{app:annot-sent}

\paragraph{Procedure.}
For each of the four benchmarks we sample 25 summaries uniformly at random
(seed $0$; deterministic, so collaborators who regenerate the tool see the
identical set). The annotator labels \emph{every} sentence of each sampled
summary against the source document. The \textit{Unclear / fragment}
option lets the annotator exclude sentence-splitting artefacts (incomplete
units lacking a verifiable claim) without forcing a label; such sentences
are dropped from $\hat\epsilon$ and $\hat p$.

\paragraph{Per-dataset statistics.}
\Table~\ref{tab:annot-stats} reports the labelling counts. The
unclear/fragment rate is 2--11\% on AggreFact-CNN, DiverSumm and HaluEval,
but \textbf{54\%} on AggreFact-Other, whose summaries come from
heterogeneous non-summarization sources (claims, QA answers, RAG outputs)
that the spaCy sentence splitter often fragments. We treat this as a
limitation of the AggreFact-Other sentence-level results.

\begin{table}[t]
\centering
\resizebox{\columnwidth}{!}{%
\begin{tabular}{lcccc}
\hline
Dataset & \#summ.\ & \#labels & \#usable & \% unclear \\
\hline
AggreFact-CNN     & 25 & 84  & 82  & 2.4\% \\
AggreFact-Other   & 25 & 61  & 28  & 54.1\% \\
DiverSumm         & 25 & 253 & 226 & 10.7\% \\
HaluEval          & 25 & 94  & 90  & 4.3\% \\
\hline
\end{tabular}
}
\caption{Sentence-level annotation statistics. The unclear / fragment
category is dropped from $\hat\epsilon$ and $\hat p$.}
\label{tab:annot-stats}
\end{table}

\subsection{Suspicious-Example Review}
\label{app:annot-susp}

\paragraph{Procedure.}
For each benchmark we sample up to 25 \emph{suspicious} summaries (human
label faithful, all five judges flag at least one sentence as unfaithful),
sorted by \texttt{doc\_id} for determinism. The annotator labels every
sentence flagged by all five judges with the same three options as above;
other sentences are shown for context but require no label. Sentences
shorter than five word tokens are excluded as fragments before labelling.

\paragraph{Per-dataset Statistics.}
\Table~\ref{tab:annot-stats-susp} reports, per benchmark, the size of the
suspicious pool, the number of summaries actually reviewed (cap of 25), and
the per-sentence verdict counts on the all-five-flagged sentences inside
those summaries. The cap is binding on AggreFact-Other and HaluEval; on
AggreFact-CNN the pool has only 15 candidates (13 after fragment filtering)
and on DiverSumm only 11 (6 after fragment filtering).

\begin{table}[t]
\centering
\resizebox{\columnwidth}{!}{%
\begin{tabular}{lcccccc}
\hline
Dataset & pool & \#summ.\ & U & F & Uncl. & \% U \\
\hline
AggreFact-CNN     &   15 & 13 & 10 & 3 & 0 & 77\% \\
AggreFact-Other   &  203 & 25 & 22 & 0 & 6 & 79\% \\
DiverSumm         &   11 &  6 &  5 & 2 & 3 & 50\% \\
HaluEval          & 1{,}444 & 25 & 24 & 1 & 2 & 89\% \\
\hline
\end{tabular}
}
\caption{Suspicious-example review counts at the \emph{sentence} level.
U = unfaithful, F = faithful, Uncl.\ = unclear (\S\ref{app:annot-susp}).
\textit{Pool} is the total number of suspicious summaries in the benchmark;
\textit{\#summ.}\ is the number of summaries actually reviewed (capped at
25). Counts on the right are over all-five-flagged sentences inside those
summaries.}
\label{tab:annot-stats-susp}
\end{table}

\subsection{Summary Composition: Human vs LLM}
\label{app:composition}

\Figure~\ref{fig:composition} classifies each summary by the joint
faithfulness of its sentences -- \emph{all unfaithful}, \emph{mixed}
($0 < K < M$), or \emph{all faithful} -- under (left) the human
sentence-level annotation of \S\ref{app:annot-sent} and (right) the mean of
the five LLM judges. Only mixed summaries discriminate $\min$ from
$\mathrm{mean}$, so the size of the orange band is the statistical power of
the $\rho(Y, A) {\le} \rho(Y, B)$ diagnostic.

\begin{figure}[t]
  \centering
  \includegraphics[width=\columnwidth]{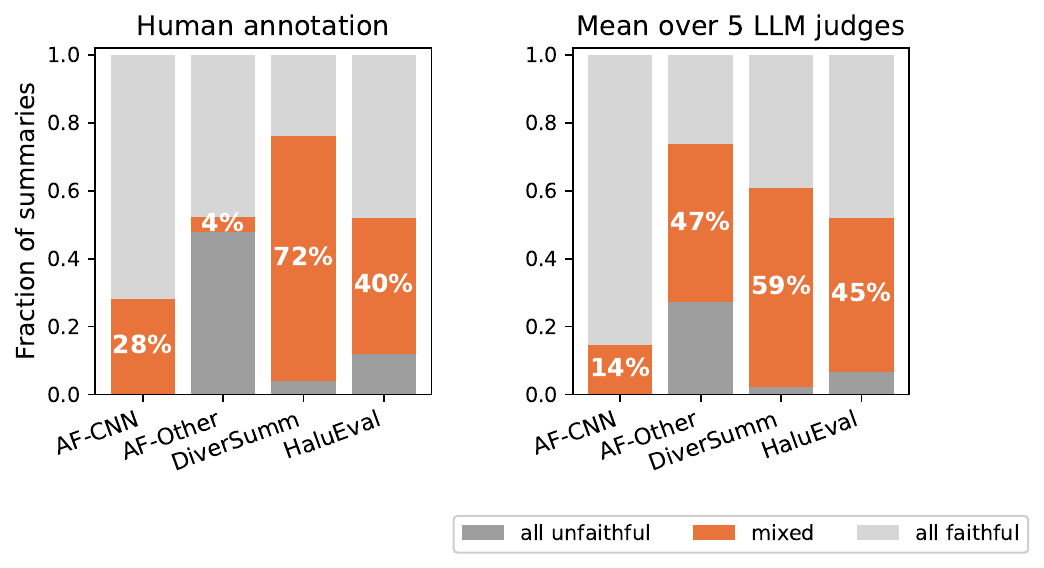}
  \caption{Stacked composition of summaries per benchmark, classified as
  \emph{all unfaithful} / \emph{mixed} / \emph{all faithful}. Left: human
  sentence-level annotation (Tool~A, 25 summaries per benchmark, unclear
  fragments dropped). Right: mean over the five LLM judges over the entire
  benchmark. Bold percentages mark the mixed segment.}
  \label{fig:composition}
\end{figure}

The two panels agree closely on DiverSumm (72\% vs 59\% mixed) and HaluEval
(40\% vs 45\%): both annotators see the same kind of within-summary
heterogeneity, and the diagnostic is well-powered on both. AggreFact-CNN
flips the direction -- humans flag 28\% of summaries as mixed but the
judges only 14\%; the LLMs miss subtle local errors there, which is
consistent with the small absolute correlation gap we observe in
\Table~\ref{tab:meta-eval-main}. AggreFact-Other is the most divergent
case: humans classify $\sim$96\% of summaries as binary (48\% all-faithful,
48\% all-unfaithful, only 4\% mixed), whereas the LLMs split 47\% as
mixed. Inspecting the disagreements (\S\ref{app:annot-susp},
\Table~\ref{tab:manual-error-cases}) suggests the LLMs are right more often
than the headline mismatch implies, but the over-segmentation effect
inflates the LLM-side mixed fraction relative to the truth. In short, the
LLM-mean mixed fraction is a workable proxy for diagnostic power, but
should be read with the human composition as a sanity check.

\section{Subclaim-Level Analysis}
\label{app:subclaim}

\subsection{Motivation}

We use sentences as natural local semantic units at our main experiments. This choice is simple
and broadly applicable. However, one sentence may contain multiple independently verifiable factual units, and a sentence-level label may conceal an error in one of these units \citep{alignscore,factscore,tang-etal-2024-minicheck,mishra2024finegrained,VeriScore,DiscInfer,Discourse,factcg}. 

We therefore repeat the analysis one level below the sentence. We treat the original human sentence label as the global label and automatically decompose the sentence into atomic subclaims, which serve as the local units. 
Under a strict conjunctive rule, a sentence is faithful only if every subclaim is supported by the source. 


\subsection{Datasets and Experimental Setup}

\paragraph{Datasets.}
We use two datasets with native human sentence-level faithfulness labels: TofuEval \citep{tang-etal-2024-tofueval}, containing 4,903 expert-labeled summary sentences, and MSumBench \citep{Multi_dimensional}, 
containing 5,367 human-labeled summary sentences. Unlike the main summary-level experiments, the labels analyzed here are attached directly to individual sentences. 

\paragraph{Subclaim decomposition.}
We decompose each summary sentence into self-contained atomic subclaims using Gemma-3-27B-IT and the atomic-fact decomposition prompt introduced by FActScore \citep{factscore}. Decomposition is performed without using the
source document, so the decomposer identifies only the propositions expressed by the summary sentence rather than evaluating their correctness. Each generated subclaim is then independently evaluated against the source document by the same five LLM judges used in the main experiments.

After decomposition, TofuEval contains 23,121 subclaims over
4,903 sentences, with 4.72 $\pm$ 1.63 subclaims per sentence;
MSumBench contains 26,979 subclaims over 5,367 sentences, with
5.03 $\pm$ 1.75 subclaims per sentence. In total,
99.4\% of the sentences contain at least two subclaims on both datasets, and we exclude the rest since a single subclaim cannot separate the $\min$ and $\mathrm{mean}$ aggregations.

\subsection{Correlation Results}

\input{tables/subclaim_results}

As shown in \Table~\ref{tab:subclaim-results}, mean aggregation correlates strongly with the human sentence label for every judge on both datasets.
The mean--minimum differences range from $0.105$ to $0.160$ on TofuEval and
from $0.029$ to $0.099$ on MSumBench. The pattern therefore occurs even when
the original annotation unit is a sentence and the local units are its atomic subclaims.
The result suggests that non-conjunctive, mean-like behavior can arise within the original human annotation unit itself.

\subsection{Manual Validation of Consensus Disagreements}

The correlation pattern could partly result from errors in subclaim decomposition or automatic subclaim verification. 
We therefore identify human-faithful sentences for which all five judges agree that at least one
subclaim is unsupported. This conservative consensus criterion flags 13.9\%
of human-faithful TofuEval sentences and 9.8\% of human-faithful MSumBench sentences.

We manually inspect 10 samples of these consensus-flagged sentences for each dataset. For each case, the annotator examines the original source, the human-labeled summary sentence, the generated subclaims, and the subclaim identified as unsupported. A case is counted as a confirmed error only when the subclaim is genuinely expressed by the sentence and is contradicted by, or unsupported by, the source.

Manual review confirms genuine violations of strict conjunction in 4 of 10 reviewed TofuEval cases and 5 of 10 reviewed MSumBench cases. Representative confirmed cases are shown in \Table~\ref{tab:subclaim-error-cases}.

\input{tables/subclaim_error_case}

These cases show that a positive sentence-level label can conceal several types of local error, including incorrect numbers, unsupported attributes, missing qualifications, overstatements, and reversals of source conditions.
The disagreement is often caused by one erroneous subclaim embedded within an otherwise faithful sentence. This pattern explains why mean aggregation may align more closely with the original sentence label than strict minimum aggregation.


\section{Full Results}
\label{app:full}

This appendix collects supporting tables and the full meta-evaluation
results. \Table~\ref{tab:dataset-subsets} breaks each of the four
benchmarks down by underlying dataset, extending \Table~\ref{tab:data}.
\Tables~\ref{tab:full-aggrefact-cnn},
\ref{tab:full-aggrefact-other-multi}, \ref{tab:full-diversumm},
and~\ref{tab:full-halueval} report the full meta-evaluation results across
all five LLM judges, four benchmarks, and the complete set of aggregation
rules, for both hard and soft sentence predictions and all four correlation
metrics. Each table corresponds to one benchmark. 

Besides minimum and mean aggregation, we also include maximum aggregation as a sanity check. Max is not a plausible faithfulness rule, but it helps rule out a degenerate explanation in which summaries are naturally homogeneous at the sentence level (all sentences are positive or negative in one sample), so that min and mean become nearly equivalent. The consistently weaker max correlations indicate that the observed mean--min pattern is not simply driven by such homogeneous examples. In stead, max matches neither the benchmark’s strict conjunctive definition nor the averaging-like behavior reflected in human labels.

\Table~\ref{tab:eps-long}
gives the per-(judge, benchmark) sentence-level error rate $\hat\epsilon$
with 95\% bootstrap CIs together with the \Section~\ref{sec:method:diagnostic} threshold
$\epsilon_0$. \Tables~\ref{tab:dataset_annotation_1}
and~\ref{tab:dataset_annotation_2} document the original annotation schemes
of the source datasets and how each was binarised for benchmark use.
Finally, \Table~\ref{tab:manual-error-cases} lists representative
manual-review errors drawn from the suspicious-summary review
(\S\ref{app:annot-susp}).

\onecolumn
\begin{table*}[t]
\centering
\small
\begin{tabular}{lrc}
\hline
Subset & Examples & Sents/summary ($\mu\pm\sigma$) \\
\hline
\multicolumn{3}{l}{\textit{AggreFact-CNN} (total 1,017)} \\
\quad SummEval & 399 & 3.24 $\pm$ 0.79 \\
\quad CLIFF & 300 & 3.34 $\pm$ 0.96 \\
\quad FRANK & 250 & 3.21 $\pm$ 0.75 \\
\quad Polytope & 68 & 3.25 $\pm$ 0.63 \\
\hline
\multicolumn{3}{l}{\textit{AggreFact-Other} (total 644)} \\
\quad ExpertQA & 357 & 2.12 $\pm$ 0.49 \\
\quad RAGTruth & 224 & 2.25 $\pm$ 0.95 \\
\quad LFQA & 19 & 2.05 $\pm$ 0.23 \\
\quad ClaimVerify & 14 & 2.07 $\pm$ 0.27 \\
\quad TofuEval-MediaSum & 11 & 2.00 $\pm$ 0.00 \\
\quad Reveal & 9 & 2.00 $\pm$ 0.00 \\
\quad TofuEval-MeetingBank & 5 & 2.00 $\pm$ 0.00 \\
\quad WiCE & 4 & 2.00 $\pm$ 0.00 \\
\quad AggreFact-XSum & 1 & 2.00 $\pm$ -- \\
\hline
\multicolumn{3}{l}{\textit{DiverSumm} (total 563)} \\
\quad GovReport & 147 & 14.95 $\pm$ 2.91 \\
\quad arXiv & 146 & 6.27 $\pm$ 2.66 \\
\quad ChemSum & 90 & 7.24 $\pm$ 2.41 \\
\quad MultiNews & 90 & 7.30 $\pm$ 3.28 \\
\quad QMSum & 90 & 3.03 $\pm$ 1.12 \\
\hline
\multicolumn{3}{l}{\textit{HaluEval} (total 20,000)} \\
\quad hallucinated & 10,000 & 3.38 $\pm$ 1.20 \\
\quad reference (faithful) & 10,000 & 3.80 $\pm$ 1.37 \\
\hline
\end{tabular}
\caption{Breakdown of the four benchmarks summarized in \Table~\ref{tab:data}. Sentence counts use the same spaCy split as the meta-evaluation. For HaluEval the two datasets are the original reference summaries and their hallucinated counterparts.}
\label{tab:dataset-subsets}
\end{table*}

\input{tables/full_aggrefact_cnn}
\input{tables/full_aggrefact_other}
\input{tables/full_diversumm}
\input{tables/full_halueval}
\input{tables/epsilon_long}
\input{tables/human_guideline_part1}
\input{tables/human_guideline_part2}

\input{tables/error_case}
\input{tables/prompt}
\end{document}

%% file: tables/subclaim_results.tex
\begin{table*}[t]
\centering
\small
\begin{tabular}{lcc@{\hspace{1.1cm}}cc}
\toprule
& \multicolumn{2}{c}{TofuEval}
& \multicolumn{2}{c}{MSumBench} \\
\cmidrule(lr){2-3}
\cmidrule(lr){4-5}
Judge
& $\min$
& $\mathrm{mean}$
& $\min$
& $\mathrm{mean}$ \\
\midrule
Qwen3-4B
& 0.292 & \textbf{0.433}
& 0.175 & \textbf{0.274} \\

Qwen3-8B
& 0.255 & \textbf{0.415}
& 0.171 & \textbf{0.256} \\

Qwen3-32B
& 0.338 & \textbf{0.467}
& 0.197 & \textbf{0.281} \\

OLMo-3-7B
& 0.195 & \textbf{0.300}
& 0.032 & \textbf{0.061} \\

Llama-3.1-8B
& 0.171 & \textbf{0.326}
& 0.088 & \textbf{0.160} \\
\bottomrule
\end{tabular}

\caption{
Pearson correlation between human sentence-level faithfulness labels and
minimum or mean aggregation of subclaim-level LLM judgments. The higher
aggregation for each judge--dataset pair is shown in bold. Mean aggregation
exceeds minimum aggregation in all ten judge--dataset comparisons.
}
\label{tab:subclaim-results}
\end{table*}

%% file: tables/subclaim_error_case.tex
\begin{table*}[t!]
\centering
\scriptsize
\setlength{\tabcolsep}{3pt}
\renewcommand{\arraystretch}{1.15}

\begin{tabular}{
    p{0.16\textwidth}
    p{0.13\textwidth}
    p{0.10\textwidth}
    p{0.53\textwidth}
}
\hline
Error type & Dataset name & Sample id & Judgment \\
\hline

\multicolumn{4}{c}{\textbf{TofuEval}} \\
\hline

Number mismatch
& TofuEval
& \texttt{004673}
& The human sentence label is faithful, but the sentence refers to
Resolution \texttt{20-0386}, whereas the source refers to Resolution
\texttt{21-0386}. The remaining content is supported, but this incorrect
number makes one subclaim unfaithful. \\

Unsupported inference
& TofuEval
& \texttt{000694}
& The sentence states that there may be a battle over ownership of the
catalog. The source only asks what will happen to the catalog and does not
mention any ownership dispute. This additional subclaim is unsupported. \\

Conditionality reversal
& TofuEval
& \texttt{000465}
& The sentence states that the United States may pursue military action
regardless of the outcome of the meetings. The source instead says that one
more attempt to restore inspections would occur before military action.
The sentence therefore reverses an important condition in the source. \\

Proposal stated as fact
& TofuEval
& \texttt{003709}
& The sentence states that the city will provide a 10.4-acre parcel for
housing. The source only reports a recommendation to approve a term sheet.
The sentence turns a pending proposal into a settled fact. \\

\hline
\multicolumn{4}{c}{\textbf{MSumBench}} \\
\hline

Unsupported attribute
& MSumBench
& \texttt{000936}
& The sentence describes PLS as a ``rare autosomal recessive disorder.''
The source supports that PLS is autosomal recessive but never states that
it is rare. The added attribute is therefore unsupported. \\

Missing qualification
& MSumBench
& \texttt{001228}
& The sentence presents the reported concentrations as generally applicable.
In the source, these concentrations apply specifically to the Bactec 460 TB
and MGIT 960 systems, while other systems use different concentrations.
Removing this condition makes the claim unfaithful. \\

Overstatement
& MSumBench
& \texttt{004012}
& The sentence states that the resolution passed with unanimous support.
The source only says that ``the ayes have it,'' which establishes majority
approval but not unanimity. \\

Numerical distortion
& MSumBench
& \texttt{003823}
& The sentence claims that sales-tax collection increased by 40\%.
The source says that collection is currently at 40\%, compared with 32\%
at the same point in the previous year. These statements have different
meanings. \\

Hallucinated content
& MSumBench
& \texttt{001278}
& The sentence introduces mass spectrometry, neuropeptides, and three named
organs. None of these appears in the source, which instead discusses the
stomatogastric ganglion and cardiac performance under acute temperature.
Most of the sentence is therefore unsupported. \\

\hline
\end{tabular}

\caption{
Manually confirmed subclaim-level errors in sentences whose original human
faithfulness labels are positive. All five LLM judges flag at least one unsupported subclaim in each case. 
}
\label{tab:subclaim-error-cases}
\end{table*}

%% file: tables/full_aggrefact_cnn.tex
\begin{longtable}{@{}llcccc@{}}
\toprule
Model & Aggregation & Pearson $r$ & Spearman $\rho$ & Kendall $\tau$ & ROC-AUC \\
\midrule
\endfirsthead

\multicolumn{6}{l}{\small\textit{Table~\thetable\ continued from previous page.}}\\
\toprule
Model & Aggregation & Pearson $r$ & Spearman $\rho$ & Kendall $\tau$ & ROC-AUC \\
\midrule
\endhead

\midrule
\multicolumn{6}{r}{\small\textit{(continued on next page)}}\\
\endfoot

\caption{Full results on AggreFact-CNN. Per-sentence scoring is either hard binary or soft logit-derived probability. Cells marked with \ns{} have a 95\% bootstrap CI that includes the null value (0 for Pearson, Spearman, Kendall; 0.5 for ROC-AUC).}
\label{tab:full-aggrefact-cnn}\\
\endlastfoot

\multicolumn{6}{c}{\textbf{Hard sentence labels}} \\
\midrule

\multirow{3}{*}{Qwen3-4B}
  & min & +0.224 & +0.224 & +0.224 & +0.572 \\
  & mean & +0.244 & +0.225 & +0.223 & +0.573 \\
  & max & +0.087 & +0.087 & +0.087 & +0.504\ns \\
\cmidrule(l){2-6}
\multirow{3}{*}{Qwen3-8B}
  & min & +0.274 & +0.274 & +0.274 & +0.629 \\
  & mean & +0.272 & +0.277 & +0.272 & +0.631 \\
  & max & +0.087 & +0.087 & +0.087 & +0.504\ns \\
\cmidrule(l){2-6}
\multirow{3}{*}{Qwen3-32B}
  & min & +0.386 & +0.386 & +0.386 & +0.673 \\
  & mean & +0.409 & +0.392 & +0.386 & +0.676 \\
  & max & +0.087 & +0.087 & +0.087 & +0.504\ns \\
\cmidrule(l){2-6}
\multirow{3}{*}{Llama-3.1-8B}
  & min & +0.250 & +0.250 & +0.250 & +0.627 \\
  & mean & +0.234 & +0.250 & +0.244 & +0.627 \\
  & max & -0.011 & -0.011 & -0.011 & +0.499\ns \\
\cmidrule(l){2-6}
\multirow{3}{*}{OLMo-3-7B}
  & min & +0.197 & +0.197 & +0.197 & +0.650 \\
  & mean & +0.211 & +0.205 & +0.191 & +0.661 \\
  & max & +0.104 & +0.104 & +0.104 & +0.518 \\
\midrule

\multicolumn{6}{c}{\textbf{Soft sentence probabilities}} \\
\midrule

\multirow{3}{*}{Qwen3-4B}
  & min & +0.239 & +0.320 & +0.266 & +0.791 \\
  & mean & +0.257 & +0.320 & +0.262 & +0.791 \\
  & max & +0.087\ns & +0.046\ns & +0.038\ns & +0.541\ns \\
\cmidrule(l){2-6}
\multirow{3}{*}{Qwen3-8B}
  & min & +0.280 & +0.304 & +0.250 & +0.776 \\
  & mean & +0.281 & +0.303 & +0.247 & +0.775 \\
  & max & +0.086\ns & +0.049\ns & +0.041\ns & +0.545\ns \\
\cmidrule(l){2-6}
\multirow{3}{*}{Qwen3-32B}
  & min & +0.405 & +0.328 & +0.269 & +0.798 \\
  & mean & +0.426 & +0.327 & +0.267 & +0.797 \\
  & max & +0.107 & +0.132 & +0.109 & +0.620 \\
\cmidrule(l){2-6}
\multirow{3}{*}{Llama-3.1-8B}
  & min & +0.276 & +0.237 & +0.194 & +0.716 \\
  & mean & +0.264 & +0.233 & +0.191 & +0.712 \\
  & max & +0.075\ns & +0.090 & +0.073 & +0.582 \\
\cmidrule(l){2-6}
\multirow{3}{*}{OLMo-3-7B}
  & min & +0.233 & +0.230 & +0.188 & +0.709 \\
  & mean & +0.238 & +0.233 & +0.190 & +0.711 \\
  & max & +0.159 & +0.159 & +0.130 & +0.644 \\
\midrule

\end{longtable}

%% file: tables/full_aggrefact_other.tex
\begin{longtable}{@{}llcccc@{}}
\toprule
Model & Aggregation & Pearson $r$ & Spearman $\rho$ & Kendall $\tau$ & ROC-AUC \\
\midrule
\endfirsthead

\multicolumn{6}{l}{\small\textit{Table~\thetable\ continued from previous page.}}\\
\toprule
Model & Aggregation & Pearson $r$ & Spearman $\rho$ & Kendall $\tau$ & ROC-AUC \\
\midrule
\endhead

\midrule
\multicolumn{6}{r}{\small\textit{(continued on next page)}}\\
\endfoot

\caption{Full results on AggreFact-Other Per-sentence scoring is either hard binary or soft logit-derived probability. Cells marked with \ns{} have a 95\% bootstrap CI that includes the null value (0 for Pearson, Spearman, Kendall; 0.5 for ROC-AUC).}
\label{tab:full-aggrefact-other-multi}\\
\endlastfoot

\multicolumn{6}{c}{\textbf{Hard sentence labels}} \\
\midrule

\multirow{3}{*}{Qwen3-4B}
  & min & +0.070\ns & +0.070\ns & +0.070\ns & +0.540\ns \\
  & mean & +0.180 & +0.177 & +0.165 & +0.626 \\
  & max & +0.232 & +0.232 & +0.232 & +0.633 \\
\cmidrule(l){2-6}
\multirow{3}{*}{Qwen3-8B}
  & min & +0.106 & +0.106 & +0.106 & +0.557 \\
  & mean & +0.208 & +0.209 & +0.195 & +0.649 \\
  & max & +0.233 & +0.233 & +0.233 & +0.644 \\
\cmidrule(l){2-6}
\multirow{3}{*}{Qwen3-32B}
  & min & +0.188 & +0.188 & +0.188 & +0.611 \\
  & mean & +0.235 & +0.236 & +0.220 & +0.670 \\
  & max & +0.213 & +0.213 & +0.213 & +0.627 \\
\cmidrule(l){2-6}
\multirow{3}{*}{Llama-3.1-8B}
  & min & +0.139 & +0.139 & +0.139 & +0.585 \\
  & mean & +0.170 & +0.152 & +0.143 & +0.607 \\
  & max & +0.166 & +0.166 & +0.166 & +0.581 \\
\cmidrule(l){2-6}
\multirow{3}{*}{OLMo-3-7B}
  & min & +0.148 & +0.148 & +0.148 & +0.587 \\
  & mean & +0.183 & +0.188 & +0.175 & +0.636 \\
  & max & +0.159 & +0.159 & +0.159 & +0.598 \\
\midrule

\multicolumn{6}{c}{\textbf{Soft sentence probabilities}} \\
\midrule

\multirow{3}{*}{Qwen3-4B}
  & min & +0.073 & +0.210 & +0.172 & +0.660 \\
  & mean & +0.181 & +0.208 & +0.170 & +0.658 \\
  & max & +0.232 & +0.259 & +0.213 & +0.698 \\
\cmidrule(l){2-6}
\multirow{3}{*}{Qwen3-8B}
  & min & +0.118 & +0.235 & +0.193 & +0.679 \\
  & mean & +0.215 & +0.250 & +0.204 & +0.690 \\
  & max & +0.240 & +0.285 & +0.233 & +0.717 \\
\cmidrule(l){2-6}
\multirow{3}{*}{Qwen3-32B}
  & min & +0.212 & +0.288 & +0.236 & +0.719 \\
  & mean & +0.264 & +0.274 & +0.224 & +0.709 \\
  & max & +0.256 & +0.276 & +0.226 & +0.710 \\
\cmidrule(l){2-6}
\multirow{3}{*}{Llama-3.1-8B}
  & min & +0.207 & +0.203 & +0.166 & +0.654 \\
  & mean & +0.225 & +0.194 & +0.158 & +0.648 \\
  & max & +0.206 & +0.166 & +0.136 & +0.627 \\
\cmidrule(l){2-6}
\multirow{3}{*}{OLMo-3-7B}
  & min & +0.201 & +0.258 & +0.211 & +0.697 \\
  & mean & +0.221 & +0.226 & +0.185 & +0.672 \\
  & max & +0.195 & +0.184 & +0.151 & +0.640 \\
\midrule

\end{longtable}

%% file: tables/full_diversumm.tex
\begin{longtable}{@{}llcccc@{}}
\toprule
Model & Aggregation & Pearson $r$ & Spearman $\rho$ & Kendall $\tau$ & ROC-AUC \\
\midrule
\endfirsthead

\multicolumn{6}{l}{\small\textit{Table~\thetable\ continued from previous page.}}\\
\toprule
Model & Aggregation & Pearson $r$ & Spearman $\rho$ & Kendall $\tau$ & ROC-AUC \\
\midrule
\endhead

\midrule
\multicolumn{6}{r}{\small\textit{(continued on next page)}}\\
\endfoot

\caption{Full results on Diversumm. Per-sentence scoring is either hard binary or soft logit-derived probability. Cells marked with \ns{} have a 95\% bootstrap CI that includes the null value (0 for Pearson, Spearman, Kendall; 0.5 for ROC-AUC).}
\label{tab:full-diversumm}\\
\endlastfoot

\multicolumn{6}{c}{\textbf{Hard sentence labels}} \\
\midrule

\multirow{3}{*}{Qwen3-4B}
  & min & +0.458 & +0.458 & +0.458 & +0.744 \\
  & mean & +0.409 & +0.457 & +0.404 & +0.764 \\
  & max & +0.071 & +0.071 & +0.071 & +0.508 \\
\cmidrule(l){2-6}
\multirow{3}{*}{Qwen3-8B}
  & min & +0.384 & +0.384 & +0.384 & +0.694 \\
  & mean & +0.413 & +0.435 & +0.374 & +0.762 \\
  & max & +0.124 & +0.124 & +0.124 & +0.523 \\
\cmidrule(l){2-6}
\multirow{3}{*}{Qwen3-32B}
  & min & +0.399 & +0.399 & +0.399 & +0.702 \\
  & mean & +0.406 & +0.436 & +0.375 & +0.763 \\
  & max & +0.090 & +0.090 & +0.090 & +0.512 \\
\cmidrule(l){2-6}
\multirow{3}{*}{Llama-3.1-8B}
  & min & +0.334 & +0.334 & +0.334 & +0.672 \\
  & mean & +0.276 & +0.335 & +0.306 & +0.679 \\
  & max & +0.044 & +0.044 & +0.044 & +0.503\ns \\
\cmidrule(l){2-6}
\multirow{3}{*}{OLMo-3-7B}
  & min & +0.153 & +0.153 & +0.153 & +0.565 \\
  & mean & +0.290 & +0.293 & +0.247 & +0.680 \\
  & max & +0.083 & +0.083 & +0.083 & +0.518 \\
\midrule

\multicolumn{6}{c}{\textbf{Soft sentence probabilities}} \\
\midrule

\multirow{3}{*}{Qwen3-4B}
  & min & +0.468 & +0.463 & +0.380 & +0.785 \\
  & mean & +0.415 & +0.464 & +0.379 & +0.785 \\
  & max & +0.069 & +0.016\ns & +0.013\ns & +0.510\ns \\
\cmidrule(l){2-6}
\multirow{3}{*}{Qwen3-8B}
  & min & +0.391 & +0.432 & +0.354 & +0.766 \\
  & mean & +0.421 & +0.439 & +0.358 & +0.770 \\
  & max & +0.132 & +0.204 & +0.170 & +0.626 \\
\cmidrule(l){2-6}
\multirow{3}{*}{Qwen3-32B}
  & min & +0.426 & +0.463 & +0.379 & +0.785 \\
  & mean & +0.418 & +0.437 & +0.357 & +0.769 \\
  & max & +0.125 & -0.035\ns & -0.029\ns & +0.479\ns \\
\cmidrule(l){2-6}
\multirow{3}{*}{Llama-3.1-8B}
  & min & +0.400 & +0.409 & +0.334 & +0.751 \\
  & mean & +0.319 & +0.377 & +0.308 & +0.732 \\
  & max & +0.074 & -0.113 & -0.092 & +0.431 \\
\cmidrule(l){2-6}
\multirow{3}{*}{OLMo-3-7B}
  & min & +0.268 & +0.376 & +0.308 & +0.731 \\
  & mean & +0.316 & +0.324 & +0.265 & +0.699 \\
  & max & +0.098 & +0.016\ns & +0.013\ns & +0.510\ns \\
\midrule

\end{longtable}

%% file: tables/full_halueval.tex
\begin{longtable}{@{}llcccc@{}}
\toprule
Model & Aggregation & Pearson $r$ & Spearman $\rho$ & Kendall $\tau$ & ROC-AUC \\
\midrule
\endfirsthead

\multicolumn{6}{l}{\small\textit{Table~\thetable\ continued from previous page.}}\\
\toprule
Model & Aggregation & Pearson $r$ & Spearman $\rho$ & Kendall $\tau$ & ROC-AUC \\
\midrule
\endhead

\midrule
\multicolumn{6}{r}{\small\textit{(continued on next page)}}\\
\endfoot

\caption{Full results on HaluEval-summarization. Per-sentence scoring is either hard binary or soft logit-derived probability. Cells marked with \ns{} have a 95\% bootstrap CI that includes the null value (0 for Pearson, Spearman, Kendall; 0.5 for ROC-AUC).}
\label{tab:full-halueval}\\
\endlastfoot

\multicolumn{6}{c}{\textbf{Hard sentence labels}} \\
\midrule

\multirow{3}{*}{Qwen3-4B}
  & min & +0.256 & +0.256 & +0.256 & +0.620 \\
  & mean & +0.306 & +0.288 & +0.269 & +0.638 \\
  & max & +0.163 & +0.163 & +0.163 & +0.528 \\
\cmidrule(l){2-6}
\multirow{3}{*}{Qwen3-8B}
  & min & +0.311 & +0.311 & +0.311 & +0.655 \\
  & mean & +0.396 & +0.382 & +0.344 & +0.707 \\
  & max & +0.241 & +0.241 & +0.241 & +0.560 \\
\cmidrule(l){2-6}
\multirow{3}{*}{Qwen3-32B}
  & min & +0.394 & +0.394 & +0.394 & +0.697 \\
  & mean & +0.472 & +0.469 & +0.421 & +0.755 \\
  & max & +0.276 & +0.276 & +0.276 & +0.575 \\
\cmidrule(l){2-6}
\multirow{3}{*}{Llama-3.1-8B}
  & min & +0.309 & +0.309 & +0.309 & +0.655 \\
  & mean & +0.400 & +0.383 & +0.345 & +0.707 \\
  & max & +0.248 & +0.248 & +0.248 & +0.560 \\
\cmidrule(l){2-6}
\multirow{3}{*}{OLMo-3-7B}
  & min & +0.138 & +0.138 & +0.138 & +0.561 \\
  & mean & +0.258 & +0.246 & +0.214 & +0.640 \\
  & max & +0.220 & +0.220 & +0.220 & +0.564 \\
\midrule

\multicolumn{6}{c}{\textbf{Soft sentence probabilities}} \\
\midrule

\multirow{3}{*}{Qwen3-4B}
  & min & +0.263 & +0.314 & +0.257 & +0.681 \\
  & mean & +0.313 & +0.347 & +0.283 & +0.700 \\
  & max & +0.168 & +0.413 & +0.342 & +0.739 \\
\cmidrule(l){2-6}
\multirow{3}{*}{Qwen3-8B}
  & min & +0.323 & +0.367 & +0.301 & +0.712 \\
  & mean & +0.404 & +0.408 & +0.333 & +0.736 \\
  & max & +0.249 & +0.399 & +0.329 & +0.730 \\
\cmidrule(l){2-6}
\multirow{3}{*}{Qwen3-32B}
  & min & +0.426 & +0.465 & +0.380 & +0.769 \\
  & mean & +0.490 & +0.491 & +0.401 & +0.783 \\
  & max & +0.294 & +0.414 & +0.339 & +0.739 \\
\cmidrule(l){2-6}
\multirow{3}{*}{Llama-3.1-8B}
  & min & +0.369 & +0.385 & +0.314 & +0.722 \\
  & mean & +0.435 & +0.416 & +0.340 & +0.740 \\
  & max & +0.304 & +0.339 & +0.277 & +0.695 \\
\cmidrule(l){2-6}
\multirow{3}{*}{OLMo-3-7B}
  & min & +0.181 & +0.234 & +0.191 & +0.635 \\
  & mean & +0.290 & +0.272 & +0.222 & +0.657 \\
  & max & +0.261 & +0.222 & +0.182 & +0.628 \\
\midrule

\end{longtable}

%% file: tables/epsilon_long.tex
\begin{longtable}{@{}llcc@{}}
\toprule
Judge & Benchmark & $\hat\epsilon$ \,[95\% CI] & $\epsilon_0$ \\
\midrule
\endfirsthead

\multicolumn{4}{l}{\small\textit{Table~\thetable\ continued from previous page.}}\\
\toprule
Judge & Benchmark & $\hat\epsilon$ \,[95\% CI] & $\epsilon_0$ \\
\midrule
\endhead

\midrule
\multicolumn{4}{r}{\small\textit{(continued on next page)}}\\
\endfoot

\caption{Per-(judge, benchmark) sentence-level error rate $\hat\epsilon$ with 95\% bootstrap CIs (resampling whole annotated summaries), and the Theorem-1 threshold $\epsilon_0$. For DiverSumm and HaluEval, $\rho(Y,\min)>\rho(Y,\mathrm{mean})$ for every $\epsilon\in(0,\tfrac12)$, so the threshold does not bind.}\label{tab:eps-long}\\
\endlastfoot

Qwen3-4B & AggreFact-CNN & 0.134\,[0.05,\,0.23] & 0.175 \\
Qwen3-4B & AggreFact-Other & 0.179\,[0.04,\,0.33] & 0.322 \\
Qwen3-4B & DiverSumm & 0.230\,[0.17,\,0.31] & $\geq$0.5 \\
Qwen3-4B & HaluEval & 0.211\,[0.13,\,0.30] & $\geq$0.5 \\
\midrule
Qwen3-8B & AggreFact-CNN & 0.134\,[0.05,\,0.23] & 0.175 \\
Qwen3-8B & AggreFact-Other & 0.214\,[0.08,\,0.35] & 0.322 \\
Qwen3-8B & DiverSumm & 0.133\,[0.09,\,0.19] & $\geq$0.5 \\
Qwen3-8B & HaluEval & 0.133\,[0.06,\,0.21] & $\geq$0.5 \\
\midrule
Qwen3-32B & AggreFact-CNN & 0.122\,[0.04,\,0.22] & 0.175 \\
Qwen3-32B & AggreFact-Other & 0.143\,[0.04,\,0.27] & 0.322 \\
Qwen3-32B & DiverSumm & 0.137\,[0.09,\,0.20] & $\geq$0.5 \\
Qwen3-32B & HaluEval & 0.100\,[0.04,\,0.17] & $\geq$0.5 \\
\midrule
Llama-3.1-8B & AggreFact-CNN & 0.159\,[0.07,\,0.26] & 0.175 \\
Llama-3.1-8B & AggreFact-Other & 0.286\,[0.11,\,0.45] & 0.322 \\
Llama-3.1-8B & DiverSumm & 0.279\,[0.22,\,0.35] & $\geq$0.5 \\
Llama-3.1-8B & HaluEval & 0.200\,[0.11,\,0.29] & $\geq$0.5 \\
\midrule
OLMo-3-7B & AggreFact-CNN & 0.159\,[0.09,\,0.24] & 0.175 \\
OLMo-3-7B & AggreFact-Other & 0.179\,[0.04,\,0.32] & 0.322 \\
OLMo-3-7B & DiverSumm & 0.230\,[0.16,\,0.31] & $\geq$0.5 \\
OLMo-3-7B & HaluEval & 0.244\,[0.17,\,0.33] & $\geq$0.5 \\
\bottomrule
\end{longtable}

%% file: tables/human_guideline_part1.tex
\begin{table*}[t]
\centering
\small
\begin{tabular}{p{0.16\linewidth}|p{0.19\linewidth}p{0.58\linewidth}}
\hline
\textbf{Benchmark} & \textbf{Dataset Name} & \textbf{Original annotation and binary label conversion} \\
\hline

AggreFact-CNN &
CLIFF~\cite{cao-wang-2021-cliff} &
The original dataset provides word/span-level annotations for intrinsic and extrinsic factual errors. The benchmark converts these local annotations into a binary sample label: a summary is inconsistent if it contains any annotated factual-error span, and consistent otherwise. \\

AggreFact-CNN &
FactCC~\cite{kryscinski-etal-2020-factcc}&
The original dataset annotates document-sentence pairs for factual consistency. The released binary consistency label is used directly as the benchmark sample label. \\

AggreFact-CNN &
Wang'20~\cite{wang-etal-2020-qags}
&
The original dataset collects sentence-level factual-consistency judgments. Later, the benchmark converts this into a strict sample label: a summary is consistent only if all of its sentences are consistent. \\


AggreFact-CNN &
SummEval~\cite{fabbri-etal-2021-summeval} &
The original dataset uses holistic 1--5 Likert scores for consistency, coherence, fluency, and relevance. SummaC converts the consistency scores into a binary label by treating a summary as consistent only if all annotators give the 5 consistency score; otherwise, it is inconsistent. AggreFact directly uses this binary label. \\

AggreFact-CNN &
Polytope~\cite{polytope} &
The original dataset uses a multidimensional error typology and severity levels. The benchmark converts these annotations into a binary factual-consistency label: a summary is inconsistent if it contains any selected factual/accuracy error type, and consistent otherwise. \\

AggreFact-CNN &
FRANK~\citep{frank} &
The original dataset provides fine-grained factual-error annotations for summaries from CNN/DailyMail and XSum. SummaC converts these annotations into a binary label by treating a summary as consistent if a majority of annotators mark it as containing no factual error. AggreFact directly uses this binary label. \\

AggreFact-CNN &
Goyal'21~\citep{goyal-durrett-2021-annotating} &
The original dataset provides span-level fine-grained factuality annotations. The benchmark converts the span-level labels into a binary sample label using an any-error rule: a summary is inconsistent if it contains annotated factual-error spans, and consistent otherwise. \\


\hline

AggreFact-Other&
TofuEval~\citep{tang-etal-2024-tofueval} &
The original dataset collects sentence-level factual-consistency labels for dialogue summaries and explicitly defines a summary as factually consistent only if all of its sentences are factually consistent. AggreFact-Other keeps the binary label. \\

AggreFact-Other&
WiCE~\citep{kamoi-etal-2023-wice} &
The original dataset provides claim/subclaim-level entailment annotations. A claim is supported only when all required subclaims are supported by the evidence. AggreFact-Other maps fully supported claims to positive labels and unsupported or partially supported claims to negative labels. \\

AggreFact-Other&
REVEAL~\citep{reveal2024} &
The original dataset annotates reasoning steps or claims for attribution. A step is fully attributable only if all information in the claim is correct according to the source. AggreFact-Other maps fully attributable claims to positive labels and other attribution labels to negative labels. \\

AggreFact-Other&
ClaimVerify~\citep{claimverify2023} &
The original dataset evaluates whether generated statements are fully supported by their cited evidence. AggreFact-Other keeps the original binary support label, where fully supported statements are positive and unsupported statements are negative. \\

AggreFact-Other&
FactCheck~\citep{factcheckgpt2023} &
The original dataset decomposes responses into atomic statements and verifies each statement independently with labels: support, partial support, refute, and irrelevant. In LLM-AggreFact's binary unification, only fully supported labels are mapped to positive; others are mapped to negative. \\

AggreFact-Other&
ExpertQA~\citep{malaviya-etal-2024-expertqa} &
The original dataset asks experts to annotate claim-level attribution and factuality. In particular, factuality asks whether every word of the claim is factually correct. AggreFact-Other maps complete or fully supported claims to positive labels and incomplete, partial, missing, or unsupported claims to negative labels. \\

AggreFact-Other&
LFQA~\citep{lfqa2023} &
The original dataset labels each answer sentence as supported, partially supported, or not supported. For binary evaluation, supported sentences are treated as positive, while partially supported and not-supported sentences are treated as negative. \\

AggreFact-Other&
RAGTruth~\citep{niu-etal-2024-ragtruth} &
The original dataset provides both response-level hallucination labels and word/span-level hallucination annotations. AggreFact-Other reformats the data into document--claim label examples and maps supported or non-hallucinated claims to positive labels, while hallucinated or unsupported claims are mapped to negative labels. \\

\hline
\end{tabular}
\caption{Original annotation forms and binary label conversion rules for Aggrefact-CNN and AggreFact-Other used in our experiments. The table separates the original annotation granularity from the later benchmark-level binary labels.}
\label{tab:dataset_annotation_1}
\end{table*}

%% file: tables/human_guideline_part2.tex
\begin{table*}[t]
\centering
\small
\begin{tabular}{p{0.16\linewidth}|p{0.19\linewidth}p{0.58\linewidth}}
\hline
\textbf{Benchmark} & \textbf{Dataset Name} & \textbf{Original annotation and binary label conversion} \\
\hline

\hline

DiverSumm &
ChemSumm~\citep{chemsum} &
The original data uses expert faithfulness annotations for generated scientific abstracts, including fine-grained error-style judgments. DiverSumm treats the annotations as numerical faithfulness scores between 0 and 1 and binarizes them by labeling a summary faithful only if the majority of annotators assign score 1; otherwise it is unfaithful. \\

DiverSumm &
MultiNews~\citep{mutinews} / QMSUM~\citep{qmsum}&

\citet{unisumm} uses holistic 1--5 human scores for dimensions such as fluency, consistency, coherence, and relevance. DiverSumm uses the consistency dimension and labels a summary faithful only when the consistency score is 5; all lower consistency scores are treated as unfaithful. \\

DiverSumm &
ArXiv~\citep{arXiv} / GovReport~\citep{govreport} &
\citet{Koh} computes factual consistency as the percentage of factually consistent summary sentences, averaged over annotators, which is average-like at the summary level. DiverSumm binarizes this score by labeling a summary faithful only if the majority of annotators assign score 1, i.e., fully consistent; otherwise it is unfaithful. \\

\hline
HaluEval~\citep{li-etal-2023-halueval} & HaluEval  &
The dataset is constructed from CNN/DailyMail document--summary pairs, where the reference summary is treated as the faithful summary and ChatGPT is prompted to generate a plausible but factually incorrect hallucinated summary. However, prior work has questioned the quality and factual consistency of CNN/DailyMail references \citep{fabbri-etal-2021-summeval,guo-etal-2022-questioning}. In particular, \citet{fabbri-etal-2021-summeval} observe that some CNN/DailyMail references contain local artifacts near the end of the reference, such as hyperlinks and clickbait-style descriptions of other articles, which are not supported by the source document. Therefore, HaluEval is useful as a controlled hallucination-recognition benchmark. In our setting, it can test whether our method detects constructed hallucinated summaries and whether it is sensitive to possible unsupported artifacts in the assumed faithful summaries. \\

\hline
\end{tabular}
\caption{Original annotation forms and binary label conversion rules for Diversum and HaluEval used in our experiments. The table separates the original annotation granularity from the later benchmark-level binary labels.}
\label{tab:dataset_annotation_2}
\end{table*}

%% file: tables/error_case.tex
\begin{table*}[t!]
\centering
\scriptsize
\setlength{\tabcolsep}{3pt}
\renewcommand{\arraystretch}{1.15}
\begin{tabular}{p{0.18\textwidth}p{0.13\textwidth}p{0.08\textwidth}p{0.53\textwidth}}
\hline
Error type & Dataset name & Sample id & Judgment\\
\hline
\multicolumn{4}{c}{\textbf{AggreFact-CNN}}\\
\hline
Boilerplate leakage   & CLIFF     & 318:s4 & \texttt{CLICK HERE for all the latest Manchester United news.} is webpage navigation, not article content.\\
Event distortion      & SummEval  & 168:s1 & The summary claims the Foxes won back-to-back against Arsenal and West Brom; the source only supports West Brom (the other reference is to Swansea/general ``back-to-back wins'').\\
Number error          & SummEval  & 105:s0 & The summary says Barcelona are six points clear; the source says the manager is six \emph{games} away from winning the title, not points.\\
\hline
\multicolumn{4}{c}{\textbf{AggreFact-Other}}\\
\hline
Entity hallucination       & ExpertQA & 7757:s2  & The summary credits \emph{Saffron Edge} for cloud services and adds a \texttt{(Passage ID 3)} citation; the source (about Miami-Dade in-house cloud) does not mention Saffron Edge.\\
Hallucinated attribution   & ExpertQA & 8262:s0  & The summary attributes Andalusian-Spanish sociolinguistic research to \emph{``Juan Andr\'es Villena Ponsoda''}; the source abstract does not name this researcher.\\
Unsupported procedure      & ExpertQA & 7801:s1  & The summary prescribes \texttt{git checkout -b} after reverting to a previous commit; the source passages discuss reverting but do not give this procedure.\\
Unsupported inference      & RAGTruth & 15949:s0 & The summary asserts that a Skype username \emph{``can be used to find their IP address''}; the passages explain how to locate a username on a profile but never connect username to IP.\\
\hline
\multicolumn{4}{c}{\textbf{DiverSumm}}\\
\hline
Boilerplate leakage   & MultiNews & 107:s9 & \texttt{click for her full column, or to read her full interview with sherman.} is web boilerplate, not summary content.\\
Reference boilerplate & GovReport & 420:s9 & \texttt{For a discussion of the process for creating a new National Park System units \ldots see CRS Report RS20158} is a CRS cross-reference, not a substantive claim.\\
Reference boilerplate & GovReport & 419:s9 & \texttt{This report will be updated as events warrant.} is CRS administrative boilerplate, repeated verbatim in another summary (420:s11).\\
Unsupported claim     & QMSum     & 207:s3 & The summary states ``screen size and speech recognition are also important features''; the meeting transcript discusses the remote control's functional design without these features.\\
\hline
\multicolumn{4}{c}{\textbf{HaluEval}}\\
\hline
Attribution distortion   & HaluEval & 234:s0 & The summary frames the article as a third-person quote (\texttt{say CNN's Carol Costello}); the source is a first-person opinion piece \emph{written by} Costello.\\
Paraphrase distortion    & HaluEval & 244:s2 & The summary asserts \emph{``Clinton needs to convince voters of her authenticity, Zelizer says''}; the source by Zelizer discusses her strengths and challenges but does not single out authenticity in those words.\\
Quoted speech fabricated & HaluEval & 302:s1 & \texttt{Danny Cevallos: Why not let people be found via social media?} attributes a rhetorical question to Cevallos that is not in the source op-ed.\\
Unsupported event detail & HaluEval & 50:s0  & The summary specifies ``a South Carolina man \ldots 66 days alone at sea''; the source intro discusses lost-at-sea stories in general and does not anchor this entity or duration.\\
\hline
\end{tabular}
\caption{Representative manual-review errors on suspicious summaries
(\Figure~\ref{fig:pies}). Sample ids use the short \texttt{doc:sent} form;
the full identifiers and source/summary text are in the released
annotations.}
\label{tab:manual-error-cases}
\end{table*}

%% file: tables/prompt.tex

\begin{table*}[t]
\centering
\small
\begin{tabular}{p{0.18\textwidth}p{0.76\textwidth}}
\toprule
\textbf{Component} & \textbf{Prompt content} \\
\midrule
\textbf{Role} &
You are a strict faithfulness evaluator for text summarization. \\

\textbf{Input} &
You are given: (1) a source document, and (2) a single sentence taken from a candidate summary of that document. \\

\textbf{Task} &
Decide whether the single sentence is faithful to the source document. A sentence is faithful only if every factual claim it makes is directly supported by, or trivially entailed by, information in the source document. If the sentence introduces information not present in the source, contradicts the source, or makes unsupported inferences, it is not faithful. \\

\textbf{Guidelines} &
Stylistic paraphrasing is acceptable; only factual content matters. If the sentence is purely a discourse connective with no factual content, e.g., ``In summary,'' label it faithful with low confidence. Be conservative: when in doubt, label the sentence as not faithful. \\

\textbf{Input format} &
\texttt{SOURCE DOCUMENT:} \newline
\texttt{<\!<\!<} \newline
\texttt{\{document\}} \newline
\texttt{>\!>\!>} \newline
\texttt{SUMMARY SENTENCE:} \newline
\texttt{<\!<\!<} \newline
\texttt{\{sentence\}} \newline
\texttt{>\!>\!>} \\

\textbf{Output format} &
Respond with strict JSON on a single line, with no Markdown or commentary: \newline
\texttt{\{"faithful": 0 or 1\}} \\
\bottomrule
\end{tabular}
\caption{Prompt used for sentence-level faithfulness evaluation.}
\label{tab:sentence-faithfulness-prompt}
\end{table*}

%% file: paper.bbl
\begin{thebibliography}{55}
\providecommand{\natexlab}[1]{#1}

\bibitem[{Adams et~al.(2023)Adams, Nguyen, Smith, Xia, Xie, Ostropolets, Deb, Chen, Naumann, and Elhadad}]{chemsum}
Griffin Adams, Bichlien Nguyen, Jake Smith, Yingce Xia, Shufang Xie, Anna Ostropolets, Budhaditya Deb, Yuan{-}Jyue Chen, Tristan Naumann, and No{\'{e}}mie Elhadad. 2023.
\newblock \href {https://doi.org/10.18653/V1/2023.ACL-LONG.587} {What are the desired characteristics of calibration sets? identifying correlates on long form scientific summarization}.
\newblock In \emph{Proceedings of the 61st Annual Meeting of the Association for Computational Linguistics (Volume 1: Long Papers), {ACL} 2023, Toronto, Canada, July 9-14, 2023}, pages 10520--10542. Association for Computational Linguistics.

\bibitem[{Bao et~al.(2025)Bao, Li, Qu, Luo, Wan, Tang, Fan, Tamber, Kazi, Sourabh, Qi, Tu, Xu, Gonzales, Mendelevitch, and Ahmad}]{FaithBench}
Forrest~Sheng Bao, Miaoran Li, Renyi Qu, Ge~Luo, Erana Wan, Yujia Tang, Weisi Fan, Manveer~Singh Tamber, Suleman Kazi, Vivek Sourabh, Mike Qi, Ruixuan Tu, Chenyu Xu, Matthew Gonzales, Ofer Mendelevitch, and Amin Ahmad. 2025.
\newblock \href {https://doi.org/10.18653/V1/2025.NAACL-SHORT.38} {Faithbench: {A} diverse hallucination benchmark for summarization by modern llms}.
\newblock In \emph{Proceedings of the 2025 Conference of the Nations of the Americas Chapter of the Association for Computational Linguistics: Human Language Technologies, {NAACL} 2025 - Volume 2: Short Papers, Albuquerque, New Mexico, April 29 - May 4, 2025}, pages 448--461. Association for Computational Linguistics.

\bibitem[{Cao and Wang(2021)}]{cao-wang-2021-cliff}
Shuyang Cao and Lu~Wang. 2021.
\newblock \href {https://doi.org/10.18653/v1/2021.emnlp-main.532} {{CLIFF}: Contrastive learning for improving faithfulness and factuality in abstractive summarization}.
\newblock In \emph{Proceedings of the 2021 Conference on Empirical Methods in Natural Language Processing}, pages 6633--6649, Online and Punta Cana, Dominican Republic. Association for Computational Linguistics.

\bibitem[{Chen et~al.(2023{\natexlab{a}})Chen, Xu, Arora, and Choi}]{lfqa2023}
Hung{-}Ting Chen, Fangyuan Xu, Shane~A. Arora, and Eunsol Choi. 2023{\natexlab{a}}.
\newblock \href {https://doi.org/10.48550/ARXIV.2310.12150} {Understanding retrieval augmentation for long-form question answering}.
\newblock \emph{CoRR}, abs/2310.12150.

\bibitem[{Chen et~al.(2023{\natexlab{b}})Chen, Liu, Xu, Yang, Zhu, Zeng, and Zhang}]{unisumm}
Yulong Chen, Yang Liu, Ruochen Xu, Ziyi Yang, Chenguang Zhu, Michael Zeng, and Yue Zhang. 2023{\natexlab{b}}.
\newblock \href {https://doi.org/10.18653/V1/2023.ACL-LONG.718} {Unisumm and summzoo: Unified model and diverse benchmark for few-shot summarization}.
\newblock In \emph{Proceedings of the 61st Annual Meeting of the Association for Computational Linguistics (Volume 1: Long Papers), {ACL} 2023, Toronto, Canada, July 9-14, 2023}, pages 12833--12855. Association for Computational Linguistics.

\bibitem[{Cohan et~al.(2018)Cohan, Dernoncourt, Kim, Bui, Kim, Chang, and Goharian}]{arXiv}
Arman Cohan, Franck Dernoncourt, Doo~Soon Kim, Trung Bui, Seokhwan Kim, Walter Chang, and Nazli Goharian. 2018.
\newblock \href {https://doi.org/10.18653/V1/N18-2097} {A discourse-aware attention model for abstractive summarization of long documents}.
\newblock In \emph{Proceedings of the 2018 Conference of the North American Chapter of the Association for Computational Linguistics: Human Language Technologies, NAACL-HLT, New Orleans, Louisiana, USA, June 1-6, 2018, Volume 2 (Short Papers)}, pages 615--621. Association for Computational Linguistics.

\bibitem[{ES et~al.(2024)ES, James, Anke, and Schockaert}]{RAGAs}
Shahul ES, Jithin James, Luis~Espinosa Anke, and Steven Schockaert. 2024.
\newblock \href {https://aclanthology.org/2024.eacl-demo.16} {Ragas: Automated evaluation of retrieval augmented generation}.
\newblock In \emph{Proceedings of the 18th Conference of the European Chapter of the Association for Computational Linguistics, {EACL} 2024 - System Demonstrations, St. Julians, Malta, March 17-22, 2024}, pages 150--158. Association for Computational Linguistics.

\bibitem[{Evans(2008)}]{evans2008dual}
Jonathan St. B.~T. Evans. 2008.
\newblock \href {https://doi.org/10.1146/annurev.psych.59.103006.093629} {Dual-processing accounts of reasoning, judgment, and social cognition}.
\newblock \emph{Annual Review of Psychology}, 59:255--278.

\bibitem[{Fabbri et~al.(2021)Fabbri, Kry{\'s}ci{\'n}ski, McCann, Xiong, Socher, and Radev}]{fabbri-etal-2021-summeval}
Alexander~R. Fabbri, Wojciech Kry{\'s}ci{\'n}ski, Bryan McCann, Caiming Xiong, Richard Socher, and Dragomir Radev. 2021.
\newblock \href {https://doi.org/10.1162/tacl_a_00373} {{S}umm{E}val: Re-evaluating summarization evaluation}.
\newblock \emph{Transactions of the Association for Computational Linguistics}, 9:391--409.

\bibitem[{Fabbri et~al.(2019)Fabbri, Li, She, Li, and Radev}]{mutinews}
Alexander~R. Fabbri, Irene Li, Tianwei She, Suyi Li, and Dragomir~R. Radev. 2019.
\newblock \href {https://doi.org/10.18653/V1/P19-1102} {Multi-news: {A} large-scale multi-document summarization dataset and abstractive hierarchical model}.
\newblock In \emph{Proceedings of the 57th Conference of the Association for Computational Linguistics, {ACL} 2019, Florence, Italy, July 28- August 2, 2019, Volume 1: Long Papers}, pages 1074--1084. Association for Computational Linguistics.

\bibitem[{Goyal and Durrett(2021)}]{goyal-durrett-2021-annotating}
Tanya Goyal and Greg Durrett. 2021.
\newblock \href {https://doi.org/10.18653/v1/2021.naacl-main.114} {Annotating and modeling fine-grained factuality in summarization}.
\newblock In \emph{Proceedings of the 2021 Conference of the North American Chapter of the Association for Computational Linguistics: Human Language Technologies}, pages 1449--1462, Online. Association for Computational Linguistics.

\bibitem[{Grattafiori et~al.(2024)Grattafiori, Dubey, Jauhri, Pandey, Kadian, Al-Dahle, Letman, Mathur, Schelten, Vaughan, Yang, Fan, Goyal, Hartshorn, Yang, Mitra, Sravankumar, Korenev, Hinsvark, Rao, Zhang, Rodriguez, Gregerson, Spataru, Roziere, Biron, Tang, Chern, Caucheteux, Nayak, Bi, Marra, McConnell, Keller, Touret, Wu, Wong, Ferrer, Nikolaidis, Allonsius, Song, Pintz, Livshits, Esiobu, Choudhary, Mahajan, Garcia-Olano, Perino, Hupkes, Lakomkin, AlBadawy, Lobanova, Dinan, Smith, Radenovic, Zhang, Synnaeve, Lee, Anderson, Thattai, Nail, Mialon, Pang, Cucurell, Nguyen, Korevaar, Xu, Touvron, Zarov, Ibarra, Kloumann, Misra, Evtimov, Zhang, Lee, Geffert, Vranes, Park, Mahadeokar, Shah, van~der Linde, Billock, Hong, Lee, Fu, Chi, Huang, Liu, Wang, Yu, Bitton, Spisak, Park, Rocca, Johnstun, Saxe, Jia, Alwala, Upasani, Plawiak, Li, Heafield, Stone, El-Arini, Iyer, Malik, Chiu, Bhalla, Rantala-Yeary, Chen, Tan, Jenkins, Martin, Madaan, Malo, Blecher, Landzaat, de~Oliveira, Muzzi, Patel, Tsimpoukelli,
  Schumacher, Liskovich, Denil, Smith, Ramesh, Baines, Srivastava, Koura, Singh, Gupta, Kambadur, Yang, Kota, Alao, Stojnic, Raileanu, Maheswari, Girdhar, Patel, Sauvestre, Polidoro, Sumbaly, Taylor, Silva, Hou, Wang, Hosseini, Chennabasappa, Singh, Bell, Kim, Edunov, Nie, Narang, Raparthy, Shen, Wan, Bhosale, Zhang, Vandenhende, Batra, Whitman, Sootla, Collot, Gururangan, Borodinsky, Herman, Fowler, Sheasha, Georgiou, Scialom, Speckbacher, Mihaylov, Xiao, Karn, Goswami, Gupta, Ramanathan, Kerkez, Gonguet, Do, Vogeti, Albiero, Petrovic, Chu, Xiong, Fu, Meers, Martinet, Wang, Tan, Xia, Xie, Jia, Wang, Goldschlag, Gaur, Babaei, Wen, Song, Zhang, Li, Mao, Coudert, Yan, Chen, Papakipos, Singh, Grizzle, Jain, Kelsey, Shajnfeld, Gangidi, Victoria, Goldstand, Menon, Sharma, Boesenberg, Vaughan, Feinstein, Kallet, Sangani, Yunus, Lupu, Alvarado, Caples, Gu, Ho, Poulton, Ryan, Ramchandani, Franco, Saraf, Chowdhury, Gabriel, Bharambe, Eisenman, Yazdan, James, Maurer, Leonhardi, Huang, Loyd, De~Paola, Paranjape, Liu,
  Wu, Ni, Hancock, Wasti, Spence, Stojkovic, Gamido, Montalvo, Parker, Burton, Mejia, Wang, Kim, Zhou, Hu, Chu, Cai, Tindal, Feichtenhofer, Gao, Civin, Beaty, Kreymer, Li, Wyatt, Adkins, Xu, Testuggine, David, Parikh, Liskovich, Foss, Wang, Le, Holland, Dowling, Jamil, Montgomery, Presani, Hahn, Wood, Brinkman, Arcaute, Dunbar, Smothers, Sun, Kreuk, Tian, Kokkinos, Ozgenel, Caggioni, Kanayet, Seide, Florez, Schwarz, Badeer, Swee, Halpern, Herman, Sizov, Guangyi, Zhang, Lakshminarayanan, Inan, Shojanazeri, Zou, Wang, Zha, Habeeb, Rudolph, Suk, Aspegren, Goldman, Zhan, Damlaj, Molybog, Tufanov, Veliche, Gat, Weissman, Geboski, Kohli, Lam, Asher, Gaya, Marcus, Tang, Chan, Zhen, Reizenstein, Teboul, Zhong, Jin, Yang, Cummings, Carvill, Shepard, McPhie, Torres, Ginsburg, Wang, Wu, U, Saxena, Prasad, Khandelwal, Zand, Matosich, Veeraraghavan, Michelena, Li, Jagadeesh, Huang, Chawla, Lakhotia, Huang, Chen, Garg, A, Silva, Bell, Zhang, Guo, Yu, Moshkovich, Wehrstedt, Khabsa, Avalani, Bhatt, Tsimpoukelli, Mankus,
  Hasson, Lennie, Reso, Groshev, Naumov, Lathi, Keneally, Liu, Seltzer, Valko, Restrepo, Patel, Vyatskov, Samvelyan, Clark, Macey, Wang, Hermoso, Metanat, Rastegari, Bansal, Santhanam, Parks, White, Bawa, Singhal, Egebo, Usunier, Mehta, Laptev, Dong, Cheng, Chernoguz, Hart, Salpekar, Kalinli, Kent, Parekh, Saab, Balaji, Rittner, Bontrager, Roux, Dollar, Zvyagina, Ratanchandani, Yuvraj, Liang, Alao, Rodriguez, Ayub, Murthy, Nayani, Mitra, Li, Hogan, Battey, Wang, Howland, Batra, Kim, Bithal, Zha, Deng, Parikh, Nezhurina, Khanna, Prasad, Yang, Raje, Guzm{\'a}n, Tiwary, Chen, Sridhar, Askham, Laskin, Kundu, Xiang, Wang, Li, Chen, Tirumala, Papi, Li, Peng, Shnayder, Panyukov, Dinkov, Lyu, Wu, Wu, Fekete, Fu, Tang, Wang, Wu, Li, Xie, Wu, Li, Weng, Zhu, Nigade, Qian, Qiao, Rait, Sukhbaatar, Liu, Cai, Wang, Song, Li, Guo, Chen, Sun, Wang, Ni, Zou, Liu, Lin, Xie, Zhang, Li, Ma, Auli, Cisse, Conneau, Kiela, Li, Chintala, Weston, Ayan, Batra, LeCun, Kavukcuoglu, Joulin, Synnaeve, Rambow, Clune, Lin, Dupoux,
  Subramanian, Honnibal, Shuster, Weston, Bordes, Weston, Lewis, Rockt{\"a}schel, Sun, Shazeer, Zhang, Devlin, Kiela, Lewis, and Zettlemoyer}]{grattafiori2024llama}
Aaron Grattafiori, Abhimanyu Dubey, Abhinav Jauhri, Abhinav Pandey, Abhishek Kadian, Ahmad Al-Dahle, Aiesha Letman, Akhil Mathur, Alan Schelten, Alex Vaughan, Amy Yang, Angela Fan, Anirudh Goyal, Anthony Hartshorn, Aobo Yang, Archi Mitra, Archie Sravankumar, Artem Korenev, Arthur Hinsvark, and 515 others. 2024.
\newblock \href {https://arxiv.org/abs/2407.21783} {The llama 3 herd of models}.
\newblock \emph{Preprint}, arXiv:2407.21783.

\bibitem[{Guo et~al.(2022)Guo, Clavel, Kamal~Eddine, and Vazirgiannis}]{guo-etal-2022-questioning}
Yanzhu Guo, Chlo{\'e} Clavel, Moussa Kamal~Eddine, and Michalis Vazirgiannis. 2022.
\newblock \href {https://doi.org/10.18653/v1/2022.emnlp-main.386} {Questioning the validity of summarization datasets and improving their factual consistency}.
\newblock In \emph{Proceedings of the 2022 Conference on Empirical Methods in Natural Language Processing}, pages 5716--5727, Abu Dhabi, United Arab Emirates. Association for Computational Linguistics.

\bibitem[{Honovich et~al.(2022)Honovich, Aharoni, Herzig, Taitelbaum, Kukliansy, Cohen, Scialom, Szpektor, Hassidim, and Matias}]{honovich-etal-2022-true-evaluating}
Or~Honovich, Roee Aharoni, Jonathan Herzig, Hagai Taitelbaum, Doron Kukliansy, Vered Cohen, Thomas Scialom, Idan Szpektor, Avinatan Hassidim, and Yossi Matias. 2022.
\newblock \href {https://doi.org/10.18653/v1/2022.naacl-main.287} {{TRUE}: Re-evaluating factual consistency evaluation}.
\newblock In \emph{Proceedings of the 2022 Conference of the North American Chapter of the Association for Computational Linguistics: Human Language Technologies}, pages 3905--3920, Seattle, United States. Association for Computational Linguistics.

\bibitem[{Huang et~al.(2020)Huang, Cui, Yang, Bao, Wang, Xie, and Zhang}]{polytope}
Dandan Huang, Leyang Cui, Sen Yang, Guangsheng Bao, Kun Wang, Jun Xie, and Yue Zhang. 2020.
\newblock \href {https://doi.org/10.18653/V1/2020.EMNLP-MAIN.33} {What have we achieved on text summarization?}
\newblock In \emph{Proceedings of the 2020 Conference on Empirical Methods in Natural Language Processing, {EMNLP} 2020, Online, November 16-20, 2020}, pages 446--469. Association for Computational Linguistics.

\bibitem[{Huang et~al.(2021)Huang, Cao, Parulian, Ji, and Wang}]{govreport}
Luyang Huang, Shuyang Cao, Nikolaus~Nova Parulian, Heng Ji, and Lu~Wang. 2021.
\newblock \href {https://doi.org/10.18653/V1/2021.NAACL-MAIN.112} {Efficient attentions for long document summarization}.
\newblock In \emph{Proceedings of the 2021 Conference of the North American Chapter of the Association for Computational Linguistics: Human Language Technologies, {NAACL-HLT} 2021, Online, June 6-11, 2021}, pages 1419--1436. Association for Computational Linguistics.

\bibitem[{Jacovi et~al.(2024)Jacovi, Bitton, Bohnet, Herzig, Honovich, Tseng, Collins, Aharoni, and Geva}]{reveal2024}
Alon Jacovi, Yonatan Bitton, Bernd Bohnet, Jonathan Herzig, Or~Honovich, Michael Tseng, Michael Collins, Roee Aharoni, and Mor Geva. 2024.
\newblock \href {https://doi.org/10.18653/v1/2024.acl-long.254} {A chain-of-thought is as strong as its weakest link: A benchmark for verifiers of reasoning chains}.
\newblock In \emph{Proceedings of the 62nd Annual Meeting of the Association for Computational Linguistics (Volume 1: Long Papers)}, pages 4615--4634, Bangkok, Thailand. Association for Computational Linguistics.

\bibitem[{Kahneman(1973)}]{kahneman1973attention}
Daniel Kahneman. 1973.
\newblock \emph{Attention and Effort}.
\newblock Prentice-Hall, Englewood Cliffs, NJ.

\bibitem[{Kamoi et~al.(2023)Kamoi, Goyal, Rodriguez, and Durrett}]{kamoi-etal-2023-wice}
Ryo Kamoi, Tanya Goyal, Juan~Diego Rodriguez, and Greg Durrett. 2023.
\newblock \href {https://doi.org/10.18653/v1/2023.emnlp-main.470} {{W}i{CE}: Real-world entailment for claims in {W}ikipedia}.
\newblock In \emph{Proceedings of the 2023 Conference on Empirical Methods in Natural Language Processing}. Association for Computational Linguistics.

\bibitem[{Koh et~al.(2022)Koh, Ju, Zhang, Liu, and Pan}]{Koh}
Huan~Yee Koh, Jiaxin Ju, He~Zhang, Ming Liu, and Shirui Pan. 2022.
\newblock \href {https://doi.org/10.18653/V1/2022.EMNLP-MAIN.172} {How far are we from robust long abstractive summarization?}
\newblock In \emph{Proceedings of the 2022 Conference on Empirical Methods in Natural Language Processing, {EMNLP} 2022, Abu Dhabi, United Arab Emirates, December 7-11, 2022}, pages 2682--2698. Association for Computational Linguistics.

\bibitem[{Kryscinski et~al.(2020)Kryscinski, McCann, Xiong, and Socher}]{kryscinski-etal-2020-factcc}
Wojciech Kryscinski, Bryan McCann, Caiming Xiong, and Richard Socher. 2020.
\newblock \href {https://doi.org/10.18653/v1/2020.emnlp-main.750} {Evaluating the factual consistency of abstractive text summarization}.
\newblock In \emph{Proceedings of the 2020 Conference on Empirical Methods in Natural Language Processing (EMNLP)}, pages 9332--9346, Online. Association for Computational Linguistics.

\bibitem[{Laban et~al.(2022)Laban, Schnabel, Bennett, and Hearst}]{summac_2022}
Philippe Laban, Tobias Schnabel, Paul~N. Bennett, and Marti~A. Hearst. 2022.
\newblock \href {https://doi.org/10.1162/TACL\_A\_00453} {Summac: Re-visiting nli-based models for inconsistency detection in summarization}.
\newblock \emph{Trans. Assoc. Comput. Linguistics}, 10:163--177.

\bibitem[{Lei et~al.(2025)Lei, Li, Li, Hu, Xu, Archer, Wang, Ching, and Deng}]{factcg}
Deren Lei, Yaxi Li, Siyao Li, Mengya Hu, Rui Xu, Ken Archer, Mingyu Wang, Emily Ching, and Alex Deng. 2025.
\newblock \href {https://doi.org/10.18653/V1/2025.NAACL-LONG.258} {Factcg: Enhancing fact checkers with graph-based multi-hop data}.
\newblock In \emph{Proceedings of the 2025 Conference of the Nations of the Americas Chapter of the Association for Computational Linguistics: Human Language Technologies, {NAACL} 2025 - Volume 1: Long Papers, Albuquerque, New Mexico, USA, April 29 - May 4, 2025}, pages 5002--5020. Association for Computational Linguistics.

\bibitem[{Li et~al.(2023)Li, Cheng, Zhao, Nie, and Wen}]{li-etal-2023-halueval}
Junyi Li, Xiaoxue Cheng, Wayne~Xin Zhao, Jian-Yun Nie, and Ji-Rong Wen. 2023.
\newblock \href {https://doi.org/10.18653/v1/2023.emnlp-main.397} {{H}alu{E}val: A large-scale hallucination evaluation benchmark for large language models}.
\newblock In \emph{Proceedings of the 2023 Conference on Empirical Methods in Natural Language Processing}. Association for Computational Linguistics.

\bibitem[{Liu et~al.(2023)Liu, Zhang, and Liang}]{claimverify2023}
Nelson Liu, Tianyi Zhang, and Percy Liang. 2023.
\newblock \href {https://doi.org/10.18653/v1/2023.findings-emnlp.467} {Evaluating verifiability in generative search engines}.
\newblock In \emph{Findings of the Association for Computational Linguistics: EMNLP 2023}, pages 7001--7025, Singapore. Association for Computational Linguistics.

\bibitem[{Malaviya et~al.(2024)Malaviya, Lee, Chen, Siegel, Yatskar, and Roth}]{malaviya-etal-2024-expertqa}
Chaitanya Malaviya, Subin Lee, Sihao Chen, Rachel Siegel, Mark Yatskar, and Dan Roth. 2024.
\newblock \href {https://aclanthology.org/2024.naacl-long.167/} {{E}xpert{QA}: Expert-curated questions and attributed answers}.
\newblock In \emph{Proceedings of the 2024 Conference of the North American Chapter of the Association for Computational Linguistics: Human Language Technologies}. Association for Computational Linguistics.

\bibitem[{Min et~al.(2025)Min, Lee, Ban, Deng, Kim, Yun, Su, Cai, and Song}]{Multi_dimensional}
Hyangsuk Min, Yuho Lee, Minjeong Ban, Jiaqi Deng, Nicole~Hee{-}Yeon Kim, Taewon Yun, Hang Su, Jason Cai, and Hwanjun Song. 2025.
\newblock \href {https://doi.org/10.18653/V1/2025.ACL-LONG.702} {Towards multi-dimensional evaluation of {LLM} summarization across domains and languages}.
\newblock In \emph{Proceedings of the 63rd Annual Meeting of the Association for Computational Linguistics (Volume 1: Long Papers), {ACL} 2025, Vienna, Austria, July 27 - August 1, 2025}, pages 14417--14450. Association for Computational Linguistics.

\bibitem[{Min et~al.(2023)Min, Krishna, Lyu, Lewis, Yih, Koh, Iyyer, Zettlemoyer, and Hajishirzi}]{factscore}
Sewon Min, Kalpesh Krishna, Xinxi Lyu, Mike Lewis, Wen{-}tau Yih, Pang~Wei Koh, Mohit Iyyer, Luke Zettlemoyer, and Hannaneh Hajishirzi. 2023.
\newblock \href {https://doi.org/10.18653/V1/2023.EMNLP-MAIN.741} {Factscore: Fine-grained atomic evaluation of factual precision in long form text generation}.
\newblock In \emph{Proceedings of the 2023 Conference on Empirical Methods in Natural Language Processing, {EMNLP} 2023, Singapore, December 6-10, 2023}, pages 12076--12100. Association for Computational Linguistics.

\bibitem[{Mishra et~al.(2024)Mishra, Asai, Balachandran, Wang, Neubig, Tsvetkov, and Hajishirzi}]{mishra2024finegrained}
Abhika Mishra, Akari Asai, Vidhisha Balachandran, Yizhong Wang, Graham Neubig, Yulia Tsvetkov, and Hannaneh Hajishirzi. 2024.
\newblock \href {https://openreview.net/forum?id=dJMTn3QOWO} {Fine-grained hallucination detection and editing for language models}.
\newblock In \emph{First Conference on Language Modeling}.

\bibitem[{Niu et~al.(2024)Niu, Wu, Zhu, Xu, Shum, Zhong, Song, and Zhang}]{niu-etal-2024-ragtruth}
Cheng Niu, Yuanhao Wu, Juno Zhu, Siliang Xu, KaShun Shum, Randy Zhong, Juntong Song, and Tong Zhang. 2024.
\newblock \href {https://doi.org/10.18653/v1/2024.acl-long.585} {{RAGT}ruth: A hallucination corpus for developing trustworthy retrieval-augmented language models}.
\newblock In \emph{Proceedings of the 62nd Annual Meeting of the Association for Computational Linguistics (Volume 1: Long Papers)}, pages 10862--10878, Bangkok, Thailand. Association for Computational Linguistics.

\bibitem[{Olmo et~al.(2026)Olmo, :, Ettinger, Bertsch, Kuehl, Graham, Heineman, Groeneveld, Brahman, Timbers, Ivison, Morrison, Poznanski, Lo, Soldaini, Jordan, Chen, Noukhovitch, Lambert, Walsh, Dasigi, Berry, Malik, Shah, Geng, Arora, Gupta, Anderson, Xiao, Murray, Romero, Graf, Asai, Bhagia, Wettig, Liu, Rangapur, Anastasiades, Huang, Schwenk, Trivedi, Magnusson, Lochner, Liu, Miranda, Sap, Morgan, Schmitz, Guerquin, Wilson, Huff, Bras, Xin, Shao, Skjonsberg, Shen, Li, Wilde, Pyatkin, Merrill, Chang, Gu, Zeng, Sabharwal, Zettlemoyer, Koh, Farhadi, Smith, and Hajishirzi}]{olmo2025olmo3}
Team Olmo, :, Allyson Ettinger, Amanda Bertsch, Bailey Kuehl, David Graham, David Heineman, Dirk Groeneveld, Faeze Brahman, Finbarr Timbers, Hamish Ivison, Jacob Morrison, Jake Poznanski, Kyle Lo, Luca Soldaini, Matt Jordan, Mayee Chen, Michael Noukhovitch, Nathan Lambert, and 50 others. 2026.
\newblock \href {https://arxiv.org/abs/2512.13961} {Olmo 3}.
\newblock \emph{Preprint}, arXiv:2512.13961.

\bibitem[{Pagnoni et~al.(2021{\natexlab{a}})Pagnoni, Balachandran, and Tsvetkov}]{pagnoni-etal-2021-understanding}
Artidoro Pagnoni, Vidhisha Balachandran, and Yulia Tsvetkov. 2021{\natexlab{a}}.
\newblock \href {https://doi.org/10.18653/v1/2021.naacl-main.383} {Understanding factuality in abstractive summarization with {FRANK}: A benchmark for factuality metrics}.
\newblock In \emph{Proceedings of the 2021 Conference of the North American Chapter of the Association for Computational Linguistics: Human Language Technologies}, pages 4812--4829, Online. Association for Computational Linguistics.

\bibitem[{Pagnoni et~al.(2021{\natexlab{b}})Pagnoni, Balachandran, and Tsvetkov}]{frank}
Artidoro Pagnoni, Vidhisha Balachandran, and Yulia Tsvetkov. 2021{\natexlab{b}}.
\newblock \href {https://doi.org/10.18653/V1/2021.NAACL-MAIN.383} {Understanding factuality in abstractive summarization with {FRANK:} {A} benchmark for factuality metrics}.
\newblock In \emph{Proceedings of the 2021 Conference of the North American Chapter of the Association for Computational Linguistics: Human Language Technologies, {NAACL-HLT} 2021, Online, June 6-11, 2021}, pages 4812--4829. Association for Computational Linguistics.

\bibitem[{Song et~al.(2024)Song, Kim, and Iyyer}]{VeriScore}
Yixiao Song, Yekyung Kim, and Mohit Iyyer. 2024.
\newblock \href {https://doi.org/10.18653/V1/2024.FINDINGS-EMNLP.552} {Veriscore: Evaluating the factuality of verifiable claims in long-form text generation}.
\newblock In \emph{Findings of the Association for Computational Linguistics: {EMNLP} 2024, Miami, Florida, USA, November 12-16, 2024}, Findings of {ACL}, pages 9447--9474. Association for Computational Linguistics.

\bibitem[{Subbiah et~al.(2024)Subbiah, Ladhak, Mishra, Adams, Chilton, and McKeown}]{STORYSUMM}
Melanie Subbiah, Faisal Ladhak, Akankshya Mishra, Griffin Adams, Lydia~B. Chilton, and Kathleen~R. McKeown. 2024.
\newblock \href {https://doi.org/10.18653/V1/2024.EMNLP-MAIN.557} {{STORYSUMM:} evaluating faithfulness in story summarization}.
\newblock In \emph{Proceedings of the 2024 Conference on Empirical Methods in Natural Language Processing, {EMNLP} 2024, Miami, FL, USA, November 12-16, 2024}, pages 9988--10005. Association for Computational Linguistics.

\bibitem[{Sweller(1988)}]{sweller1988cognitive}
John Sweller. 1988.
\newblock \href {https://doi.org/10.1207/s15516709cog1202_4} {Cognitive load during problem solving: Effects on learning}.
\newblock \emph{Cognitive Science}, 12(2):257--285.

\bibitem[{Tang et~al.(2023)Tang, Goyal, Fabbri, Laban, Xu, Yavuz, Kryscinski, Rousseau, and Durrett}]{tang-etal-2023-understanding}
Liyan Tang, Tanya Goyal, Alex Fabbri, Philippe Laban, Jiacheng Xu, Semih Yavuz, Wojciech Kryscinski, Justin Rousseau, and Greg Durrett. 2023.
\newblock \href {https://doi.org/10.18653/v1/2023.acl-long.650} {Understanding factual errors in summarization: Errors, summarizers, datasets, error detectors}.
\newblock In \emph{Proceedings of the 61st Annual Meeting of the Association for Computational Linguistics (Volume 1: Long Papers)}, pages 11626--11644, Toronto, Canada. Association for Computational Linguistics.

\bibitem[{Tang et~al.(2024{\natexlab{a}})Tang, Laban, and Durrett}]{llmaggrefact_dataset}
Liyan Tang, Philippe Laban, and Greg Durrett. 2024{\natexlab{a}}.
\newblock \href {https://doi.org/10.18653/v1/2024.emnlp-main.499} {{M}ini{C}heck: Efficient fact-checking of {LLM}s on grounding documents}.
\newblock In \emph{Proceedings of the 2024 Conference on Empirical Methods in Natural Language Processing}, pages 8818--8847, Miami, Florida, USA. Association for Computational Linguistics.

\bibitem[{Tang et~al.(2024{\natexlab{b}})Tang, Laban, and Durrett}]{tang-etal-2024-minicheck}
Liyan Tang, Philippe Laban, and Greg Durrett. 2024{\natexlab{b}}.
\newblock \href {https://doi.org/10.18653/v1/2024.emnlp-main.499} {{M}ini{C}heck: Efficient fact-checking of {LLM}s on grounding documents}.
\newblock In \emph{Proceedings of the 2024 Conference on Empirical Methods in Natural Language Processing}, pages 8818--8847, Miami, Florida, USA. Association for Computational Linguistics.

\bibitem[{Tang et~al.(2024{\natexlab{c}})Tang, Shalyminov, Wong, Burnsky, Vincent, Yang, Singh, Feng, Song, Su, Sun, Zhang, Mansour, and McKeown}]{tang-etal-2024-tofueval}
Liyan Tang, Igor Shalyminov, Amy Wing-mei Wong, Jon Burnsky, Jake~W. Vincent, Yu'an Yang, Siffi Singh, Song Feng, Hwanjun Song, Hang Su, Lijia Sun, Yi~Zhang, Saab Mansour, and Kathleen McKeown. 2024{\natexlab{c}}.
\newblock \href {https://doi.org/10.18653/v1/2024.naacl-long.251} {{T}ofu{E}val: Evaluating hallucinations of {LLM}s on topic-focused dialogue summarization}.
\newblock In \emph{Proceedings of the 2024 Conference of the North American Chapter of the Association for Computational Linguistics: Human Language Technologies}. Association for Computational Linguistics.

\bibitem[{Tang et~al.(2022)Tang, Fabbri, Li, Mao, Adams, Wang, Celikyilmaz, Mehdad, and Radev}]{Protocols}
Xiangru Tang, Alexander~R. Fabbri, Haoran Li, Ziming Mao, Griffin Adams, Borui Wang, Asli Celikyilmaz, Yashar Mehdad, and Dragomir~R. Radev. 2022.
\newblock \href {https://doi.org/10.18653/V1/2022.NAACL-MAIN.417} {Investigating crowdsourcing protocols for evaluating the factual consistency of summaries}.
\newblock In \emph{Proceedings of the 2022 Conference of the North American Chapter of the Association for Computational Linguistics: Human Language Technologies, {NAACL} 2022, Seattle, WA, United States, July 10-15, 2022}, pages 5680--5692. Association for Computational Linguistics.

\bibitem[{Tversky and Kahneman(1974)}]{tversky1974judgment}
Amos Tversky and Daniel Kahneman. 1974.
\newblock \href {https://doi.org/10.1126/science.185.4157.1124} {Judgment under uncertainty: Heuristics and biases}.
\newblock \emph{Science}, 185(4157):1124--1131.

\bibitem[{Wan et~al.(2025)Wan, Vig, Bansal, and Joty}]{Positional_Bias}
David Wan, Jesse Vig, Mohit Bansal, and Shafiq Joty. 2025.
\newblock \href {https://doi.org/10.18653/V1/2025.NAACL-LONG.442} {On positional bias of faithfulness for long-form summarization}.
\newblock In \emph{Proceedings of the 2025 Conference of the Nations of the Americas Chapter of the Association for Computational Linguistics: Human Language Technologies, {NAACL} 2025 - Volume 1: Long Papers, Albuquerque, New Mexico, USA, April 29 - May 4, 2025}, pages 8791--8810. Association for Computational Linguistics.

\bibitem[{Wang et~al.(2020)Wang, Cho, and Lewis}]{wang-etal-2020-qags}
Alex Wang, Kyunghyun Cho, and Mike Lewis. 2020.
\newblock \href {https://doi.org/10.18653/v1/2020.acl-main.450} {Asking and answering questions to evaluate the factual consistency of summaries}.
\newblock In \emph{Proceedings of the 58th Annual Meeting of the Association for Computational Linguistics}, pages 5008--5020, Online. Association for Computational Linguistics.

\bibitem[{Wang et~al.(2026)Wang, Prei{\ss}, Bugue{\~{n}}o, Hoffbauer, Ghajar, Buz, and de~Melo}]{ReFACT}
Yindong Wang, Martin Prei{\ss}, Margarita Bugue{\~{n}}o, Jan~Vincent Hoffbauer, Abdullatif Ghajar, Tolga Buz, and Gerard de~Melo. 2026.
\newblock \href {https://doi.org/10.18653/V1/2026.EACL-LONG.381} {Refact: {A} benchmark for scientific confabulation detection with positional error annotations}.
\newblock In \emph{Proceedings of the 19th Conference of the European Chapter of the Association for Computational Linguistics, {EACL} 2026 - Volume 1: Long Papers, Rabat, Morocco, March 24-29, 2026}, pages 8174--8187. Association for Computational Linguistics.

\bibitem[{Wang et~al.(2024)Wang, Gangi~Reddy, Mujahid, Arora, Rubashevskii, Geng, Mohammed~Afzal, Pan, Borenstein, Pillai, Augenstein, Gurevych, and Nakov}]{factcheckgpt2023}
Yuxia Wang, Revanth Gangi~Reddy, Zain~Muhammad Mujahid, Arnav Arora, Aleksandr Rubashevskii, Jiahui Geng, Osama Mohammed~Afzal, Liangming Pan, Nadav Borenstein, Aditya Pillai, Isabelle Augenstein, Iryna Gurevych, and Preslav Nakov. 2024.
\newblock \href {https://doi.org/10.18653/v1/2024.findings-emnlp.830} {Factcheck-bench: Fine-grained evaluation benchmark for automatic fact-checkers}.
\newblock In \emph{Findings of the Association for Computational Linguistics: EMNLP 2024}, pages 14199--14230, Miami, Florida, USA. Association for Computational Linguistics.

\bibitem[{Wei et~al.(2024)Wei, Yang, Song, Lu, Hu, Huang, Tran, Peng, Liu, Huang, Du, and Le}]{Long_form}
Jerry Wei, Chengrun Yang, Xinying Song, Yifeng Lu, Nathan Hu, Jie Huang, Dustin Tran, Daiyi Peng, Ruibo Liu, Da~Huang, Cosmo Du, and Quoc~V. Le. 2024.
\newblock \href {http://papers.nips.cc/paper\_files/paper/2024/hash/937ae0e83eb08d2cb8627fe1def8c751-Abstract-Conference.html} {Long-form factuality in large language models}.
\newblock In \emph{Advances in Neural Information Processing Systems 38: Annual Conference on Neural Information Processing Systems 2024, NeurIPS 2024, Vancouver, BC, Canada, December 10 - 15, 2024}.

\bibitem[{Yang et~al.(2025)Yang, Li, Yang, Zhang, Hui, Zheng, Yu, Gao, Huang, Lv, Zheng, Liu, Zhou, Huang, Hu, Ge, Wei, Lin, Tang, Yang, Tu, Zhang, Yang, Yang, Zhou, Zhou, Lin, Dang, Bao, Yang, Yu, Deng, Li, Xue, Li, Zhang, Wang, Zhu, Men, Gao, Liu, Luo, Li, Tang, Yin, Ren, Wang, Zhang, Ren, Fan, Su, Zhang, Zhang, Wan, Liu, Wang, Cui, Zhang, Zhou, and Qiu}]{qwen3}
An~Yang, Anfeng Li, Baosong Yang, Beichen Zhang, Binyuan Hui, Bo~Zheng, Bowen Yu, Chang Gao, Chengen Huang, Chenxu Lv, Chujie Zheng, Dayiheng Liu, Fan Zhou, Fei Huang, Feng Hu, Hao Ge, Haoran Wei, Huan Lin, Jialong Tang, and 41 others. 2025.
\newblock \href {https://arxiv.org/abs/2505.09388} {Qwen3 technical report}.
\newblock \emph{Preprint}, arXiv:2505.09388.

\bibitem[{Yin et~al.(2026)Yin, Li, Liu, Wang, Song, Ma, Liu, Indurthi, Deng, He, Zhou, and Wang}]{ContextCheck}
Yueqin Yin, Yaxi Li, Xin Liu, Xun Wang, Kaiqiang Song, Simin Ma, Shujian Liu, Sathish~Reddy Indurthi, Haoyun Deng, Pengcheng He, Mingyuan Zhou, and Song Wang. 2026.
\newblock \href {https://aclanthology.org/2026.findings-acl.658/} {Contextcheck: Sentence-level faithfulness verification with context-aware disambiguation}.
\newblock In \emph{Findings of the Association for Computational Linguistics, {ACL} 2026, San Diego, California, United States, July 2-7, 2026}, pages 13419--13468. Association for Computational Linguistics.

\bibitem[{Zha et~al.(2023)Zha, Yang, Li, and Hu}]{alignscore}
Yuheng Zha, Yichi Yang, Ruichen Li, and Zhiting Hu. 2023.
\newblock \href {https://doi.org/10.18653/V1/2023.ACL-LONG.634} {Alignscore: Evaluating factual consistency with {A} unified alignment function}.
\newblock In \emph{Proceedings of the 61st Annual Meeting of the Association for Computational Linguistics (Volume 1: Long Papers), {ACL} 2023, Toronto, Canada, July 9-14, 2023}, pages 11328--11348. Association for Computational Linguistics.

\bibitem[{Zhang et~al.(2024)Zhang, Xu, and Perez-Beltrachini}]{zhang-etal-2024-fine}
Huajian Zhang, Yumo Xu, and Laura Perez-Beltrachini. 2024.
\newblock \href {https://doi.org/10.18653/v1/2024.eacl-long.102} {Fine-grained natural language inference based faithfulness evaluation for diverse summarisation tasks}.
\newblock In \emph{Proceedings of the 18th Conference of the European Chapter of the Association for Computational Linguistics (Volume 1: Long Papers)}, pages 1701--1722, St. Julian{'}s, Malta. Association for Computational Linguistics.

\bibitem[{Zhang et~al.(2025)Zhang, Balalau, and Manolescu}]{DiscInfer}
Kun Zhang, Oana Balalau, and Ioana Manolescu. 2025.
\newblock \href {https://aclanthology.org/2025.findings-acl.46/} {Structured discourse representation for factual consistency verification}.
\newblock In \emph{Findings of the Association for Computational Linguistics, {ACL} 2025, Vienna, Austria, July 27 - August 1, 2025}, Findings of {ACL}, pages 820--838. Association for Computational Linguistics.

\bibitem[{Zhong et~al.(2021)Zhong, Yin, Yu, Zaidi, Mutuma, Jha, Awadallah, Celikyilmaz, Liu, Qiu, and Radev}]{qmsum}
Ming Zhong, Da~Yin, Tao Yu, Ahmad Zaidi, Mutethia Mutuma, Rahul Jha, Ahmed~Hassan Awadallah, Asli Celikyilmaz, Yang Liu, Xipeng Qiu, and Dragomir~R. Radev. 2021.
\newblock \href {https://doi.org/10.18653/V1/2021.NAACL-MAIN.472} {Qmsum: {A} new benchmark for query-based multi-domain meeting summarization}.
\newblock In \emph{Proceedings of the 2021 Conference of the North American Chapter of the Association for Computational Linguistics: Human Language Technologies, {NAACL-HLT} 2021, Online, June 6-11, 2021}, pages 5905--5921. Association for Computational Linguistics.

\bibitem[{Zhong and Litman(2025{\natexlab{a}})}]{Discourse}
Yang Zhong and Diane~J. Litman. 2025{\natexlab{a}}.
\newblock \href {https://doi.org/10.18653/V1/2025.NAACL-LONG.103} {Discourse-driven evaluation: Unveiling factual inconsistency in long document summarization}.
\newblock In \emph{Proceedings of the 2025 Conference of the Nations of the Americas Chapter of the Association for Computational Linguistics: Human Language Technologies, {NAACL} 2025 - Volume 1: Long Papers, Albuquerque, New Mexico, USA, April 29 - May 4, 2025}, pages 2050--2073. Association for Computational Linguistics.

\bibitem[{Zhong and Litman(2025{\natexlab{b}})}]{Tale}
Yang Zhong and Diane~J. Litman. 2025{\natexlab{b}}.
\newblock \href {https://aclanthology.org/2025.findings-acl.648/} {A tale of evaluating factual consistency: Case study on long document summarization evaluation}.
\newblock In \emph{Findings of the Association for Computational Linguistics, {ACL} 2025, Vienna, Austria, July 27 - August 1, 2025}, Findings of {ACL}, pages 12511--12532. Association for Computational Linguistics.

\end{thebibliography}
